\documentclass[journal]{IEEEtran}
\usepackage{amsmath,amssymb,amsfonts}
\usepackage{algorithmic}
\usepackage{algorithm}
\usepackage{array}
\usepackage{textcomp}
\usepackage{stfloats}
\usepackage{url}
\usepackage{verbatim}
\usepackage{amsmath}
\usepackage{cite}
\usepackage{xcolor}
\usepackage{mathtools}

\newcommand{\HL}[1]{{\color{black}{#1}}} %%%
\newcommand{\XY}[1]{{\color{black}{#1}}} %%%
\newcommand{\Yang}[1]{{\color{black}{#1}}} %%%
\newcommand{\ADD}[1]{{\color{black}{#1}}}
\newtheorem{proof}{\textbf{Proof}}

\newtheorem{lem}{\textbf{Lemma}}

\newtheorem{theorem}{\textbf{Theorem}}
\newtheorem{remark}{\textbf{Remark}}

\newtheorem{definition}{\textbf{Definition}}
\usepackage{graphicx}
\usepackage{float}
\usepackage{subfigure}
\usepackage{enumitem}
\usepackage{tikz}
\usepackage{ulem}

\begin{document}

% \title{Upper-Bound Leveraged Utility-Privacy Bi-Objective Optimization in Federated Learning}
\title{A Theoretical Analysis of \XY{Efficiency Constrained} Utility-Privacy Bi-Objective Optimization in Federated Learning}

\author{Hanlin~Gu\IEEEauthorrefmark{1} ,
Xinyuan Zhao\IEEEauthorrefmark{1},
Gongxi Zhu,
Yuxing Han,
Yan Kang,
Lixin~Fan,~\IEEEmembership{Member,~IEEE,}
and~Qiang~Yang,~\IEEEmembership{Fellow,~IEEE}

\thanks{\IEEEauthorrefmark{1} Hanlin Gu and Xinyuan Zhao contribute equally in this paper}
\thanks{\IEEEauthorrefmark{2} Yuxing Han is the corresponding author}
\thanks{This work was supported by the National Natural Science Foundation of China (NO.62206154) and Shenzhen Startup Funding (No. QD2023014C).}}
% \thanks{This work was supported by the National Science Foundation of China (NO.62206154), Shenzhen Startup Funding (No. 2023Z237).}

% \author{\IEEEauthorblockN{1\textsuperscript{st} Hanlin Gu},
% \IEEEauthorblockA{\textit{AI Group},
% \textit{Webank},
% China,
% hanlingu@webank.com}\\
% \and
% \IEEEauthorblockN{1\textsuperscript{st} Xinyuan Zhao},
% \IEEEauthorblockA{\textit{Shenzhen International Graduate School},
% \textit{Tsinghua University},
% China,
% zhao-xy23@mails.tsinghua.edu.cn}\\
% \and
% % \IEEEauthorblockN{3\textsuperscript{th} Yan Kang},
% % \IEEEauthorblockA{\textit{AI Group},
% % \textit{Webank},
% % China,
% % yangkang@webank.com}\\
% \and
% \IEEEauthorblockN{3\textsuperscript{th} Yuxing Han},
% \IEEEauthorblockA{\textit{Shenzhen International Graduate School},
% \textit{Tsinghua University},
% China,
% yuxinghan@sz.tsinghua.edu.cn}

% \thanks{\IEEEauthorrefmark{4} Yuxing Han is corresponding author.}% <-this % stops a space



\IEEEtitleabstractindextext{\begin{abstract}
Federated learning (FL) enables multiple clients to collaboratively learn a shared model without sharing their individual data. Concerns about utility, privacy, and training efficiency in FL have garnered significant research attention. Differential privacy has emerged as a prevalent technique in FL, safeguarding the privacy of individual user data while impacting utility and training efficiency. Within Differential Privacy Federated Learning (DPFL), previous studies have primarily focused on the utility-privacy trade-off, neglecting training efficiency, which is crucial for timely completion. Moreover, differential privacy achieves privacy by introducing controlled randomness (noise) on selected clients in each communication round. Previous work has mainly examined the impact of noise level ($\sigma$) and communication rounds ($T$) on the privacy-utility dynamic, overlooking other influential factors like the sample ratio ($q$, the proportion of selected clients). This paper systematically formulates an efficiency-constrained utility-privacy bi-objective optimization problem in DPFL, focusing on $\sigma$, $T$, and $q$. We provide a comprehensive theoretical analysis, yielding analytical solutions for the Pareto front. Extensive empirical experiments verify the validity and efficacy of our analysis, offering valuable guidance for low-cost parameter design in DPFL.

% an efficient tool for parameter design in a federated learning system.
% This formulation is meticulously designed to uncover the Pareto front, which characterizes the optimal balance between the noise level, the global learning epoch, and the sample ratio. 
% \XY{For privacy leakage}, differential privacy has emerged as a prevalent technique in FL to safeguard the privacy of individual user data. 
% This is accomplished through the introduction of controlled randomness, often manifested as noise, during the model training phase. 
% Within the realm of Differential Privacy Federated Learning (DPFL), the introduction of differential privacy leads to an inherent trade-off between privacy preservation and model utility, e.g., adding excessive noise would enhance privacy while resulting in a degradation of model utility. 
% \XY{Previous studies predominantly concentrated on the effect of the noise level ($\sigma$)}

% \XY{Previous studies predominantly concentrated on the utility-privacy trade-off, but ignores the training efficiency}. \XY{Moreover, focusing on the effect of the noise level ($\sigma$) on the dynamic between privacy and utility, previous studies} ignore other influential factors such as the sample ratio ($q$).

\end{abstract}

}

\maketitle

\IEEEdisplaynontitleabstractindextext

\IEEEpeerreviewmaketitle
\section{Introduction}
The escalating stringency of legal and regulatory parameters, exemplified by initiatives like GDPR\footnote{GDPR is applicable as of May 25th, 2018 in all European member states to harmonize data privacy laws across Europe. \url{https://gdpr.eu/}} and HIPAA\footnote{HIPAA is a federal law of the USA created in 1996. It required the creation of national standards to protect sensitive patient health information from being disclosed without the patient’s consent or knowledge}, has imposed rigorous constraints on user privacy. This has led to a situation where the amalgamation of private data from distinct users or organizations for the purpose of training machine learning models is no longer allowed. Federated learning \cite{mcmahan2017communicationefficient, yang2019federated} is a pioneering paradigm in machine learning that addresses the challenge of training models on decentralized data sources while respecting stringent privacy and security constraints. In federated learning, multiple participating entities or clients collaboratively train a shared machine learning model without directly sharing their raw data. 

Recently, \XY{apart from traditional utility issue, privacy and efficiency} issues in federated learning attract wide research attention.
\XY{In order to prevent privacy leakage from the exchanged information in FL, differential privacy has been proposed as an important privacy protection mechanism \cite{dwork2014algorithmic, abadi2016deeplearningwithDP}, which} is achieved by introducing noise into the exchanged gradients in the training process.
% which plays a crucial role in enhancing the privacy of participants' data \cite{dwork2014algorithmic, abadi2016deeplearningwithDP}.
In differential privacy federated learning (DPFL), there is a trade-off between utility and privacy as demonstrated in \cite{yang2019federatedconceptsandapplications, zhang2022nofreelunch, kang2023optimizing}. 
For instance, it was shown that attackers may infer training images at pixel level accuracy even random noise are added to exchanged gradients \cite{he2019model,zhu2019deep,geiping2020inverting,yin2021see,zhao2020idlg}, but exceedingly large noise jeopardise learning reliability and lead to significant degradation of utility \cite{yang2019federatedconceptsandapplications, zhang2022nofreelunch, kang2023optimizing}.
\XY{On the other hand, training efficiency which descries the global training time is also an important concern in the federated learning framework \cite{konevcny2016federated,singh2019detailed,ghosh2020efficient}. Ignoring the training efficiency may lead to exceedingly long training time outside an acceptable timeframe.}

\XY{A series of work attempts to balance utility loss and privacy leakage in DPFL, but neglect the training efficiency.}
\XY{Specifically, these work provides different strategies to search for better parameters to look for optimal privacy-utility trade-off in DPFL.}
Some work \cite{geyer2017differentially,wei2020federated} aimed to minimize the utility loss while accounting for the resulting privacy leakage within the confines of a specified privacy budget ($\epsilon_0$). 
% These efforts have not exhaustively delineated the entire spectrum of the optimal trade-off, often referred to as the Pareto front\footnote{The definition of Pareto front is given in Part C: Bi-Objective Optimization, II. Related Work and Background}. 
Another line of work \cite{yang2019federatedconceptsandapplications,wu2020theoretical} firstly trained the model until convergence with different noise extent. 
Subsequently, a comparison is drawn between resulting privacy leakages to identify the parameter settings that yield the least privacy leakage. 
\XY{However, the existing methods ignore the training efficiency in DPFL, which refers to the global training time (discussed in Sec. \ref{sec:Privacy Leakage and Utility Loss in DPFL} in detail).}

% With approximately equal training time each round in DPFL, the global training efficiency can be directly accounted by the global training time ($T$).}
\XY{Moreover, in the process of searching optimal parameters, the traditional methods primarily concentrate on the influence of parameters as noise level ($\sigma$) and communication rounds ($T$), but ignore sample ratio ($q$) as another important factor}. 
Kang et al. \cite{kang2023optimizing} manipulate the noise level ($\sigma$) to simultaneously enhance privacy, minimize utility and efficiency. Other works \cite{geyer2017differentially,wei2020federated,yang2019federatedconceptsandapplications,wu2020theoretical} put forth different strategies to search for noise level ($\sigma$) and communication rounds ($T$) with aim of achieving optimal privacy-utility trade-off in DPFL.
These methods neglects sample ratio (the ratio of participating clients in each round among all the $K$ clients, denoted as $q$), which has significant influence on both the utility loss \cite{fraboni2022on, cho2020client} and privacy leakage \cite{abadi2016deeplearningwithDP}. \XY{To look for the optimal trade-off by considering the influence of sample ratio ($q$), a naive expansion of current methods as iterating over $q$ suffers from really high computation cost.}

In this work, we formulate a constrained bi-objective optimization problem in Sect. \ref{Section: MOO in DPFL} with aim of \XY{ensuring acceptable training efficiency} and \XY{reducing optimal parameter searching computation cost}.
\XY{This formulation (Eq. \ref{equation: formulation_3factor}) focuses on minimizing the privacy leakage and utility loss, and include an upper constraint of training efficiency to ensure acceptable training time in DPFL.}
\XY{It} identifies the Pareto front encompassing the noise level ($\sigma$), the communication rounds ($T$) and sample ratio ($q$). Moreover, we conduct theoretical analysis to offer insights into the interplay among various parameters in DPFL. Detailed theoretical findings can be found in Thm. \ref{Thm: analytical solution} of Sect. \ref{Section: theoretical analyzing}.
Notably, the Pareto \XY{optimal solutions} for the constrained bi-objective optimization problem in DPFL adheres to the relationship $k\sigma^2T = qK$ ($K$ denotes the total number of clients; $k$ is a constant), \XY{which serves as a powerful tool to help with low cost parameter design in DPFL discussed in Sec. \ref{sec:parameter design}.} 
Finally, experimental results in Sect. \ref{Section: experiment} also verify the theoretical analysis on Pareto optimal solutions\footnote{The definition of Pareto optimal solutions are given in Part C: Bi-Objective Optimization, II. Related Work and Background} on MNIST dataset (with logistic regression and LeNet) and CIFAR-10 dataset (with ResNet-18).
Our main contribution is summarized as following:

\begin{itemize}
\item We formulate the utility loss and privacy leakage \XY{with training efficiency upper constraint} in differential privacy federated learning (DPFL) as constrained bi-objective optimization problem with respect to noise level ($\sigma$), communication rounds ($T$) and sample ratio ($q$). 
% We expand the decision searching space from only one dimension to multi dimensions among communication rounds $T$, standard derivation of added noise $\sigma$, and sample ratio $v$, which helps achieve better trade-off.
% \item We theoretically elucidate the analytical relationship among Pareto optimal solutions $\sigma$, $T$ and $q$ under different participant numbers $K$ by leveraging the upper bound of utility loss and privacy leakage.
\item We theoretically elucidate the analytical Pareto optimal solutions of the \XY{constrained} bi-objective optimization problem in DPFL w.r.t. $\sigma$, $T$ and $q$ under different participant numbers $K$.
% Under same sample ratio, $\sigma^2$ and $T$ has inverse proportional relationship. 
% Comparing different sample ratio, The optimal Pareto $\sigma^2$ has proportional relationship with $v$.
\item The experiments on \ADD{MNIST and CIFAR-10} verify the correctness of our \XY{theoretical analysis} of the \XY{constrained} bi-objective optimization problem in DPFL. Moreover, it illustrates the theoretical \XY{analysis} can guide clients to design the effective parameters in DPFL with much lower computation cost than traditional methods.
% We demonstrate the analytical solution by experiments using both convolutional neural network and multi-layer perception neural network. 
% The experiments clearly shows the iverse proportional relationship between $\sigma^2$ and $T$.
% We can also observe the proportional relationship between $\sigma^2 T$ and sample ratio.
% Furthermore, we also use a simple example to demonstrate how the derived relationships help with the parameter design in DPFL.
\end{itemize}

\begin{table}[htbp]
\caption{Table of Notation}
    \begin{center}
        \begin{tabular}{lccc}
            \hline
            Notation & Meaning \\ 
            \hline 
            $\sigma$ & noise level \\
            $\sigma_{max}$ & upper constraint of $\sigma$ \\
            $q$ & sample ratio \\
            $T$ & communication rounds \\
            $T_{max}$ & maximum of communication round $t$ \\
            $E$ & local training epochs within each round \\
            $K$ & total number of clients \\
            $\epsilon_0$ & privacy budget \\
            $D_k$ & private dataset of client $k$ \\
            $D_{test}$ & test dataset \\
            $w^t$ & global model at round $t$ \\
            $\Tilde{\Delta} w_k^t$ & protected model gradients of client $k$ at round $t$ \\
            $w_k^{t,e}$ & local model weight of client $k$ at round $t$ \\
            &and local epochs $e$ \\
            $\Delta w_k^t$ & model gradients of client $k$ at round $t$ \\
            $n_k^t$ & Gaussian noise of client $k$ at round $t$ \\
            $c_{clip}$ & clipping constant \\
            $P_t$ & participating clients at round $t$ \\
            $\eta$ & learning rate \\
            $B$ & batch size \\
            $(x,y)$ & feature-label pair in dataset \\
            $L_{ce}$ & cross-entropy loss function \\
            $\epsilon_p$ & privacy leakage \\
            $\epsilon_u$ & utility loss \\
            $\epsilon_e$ & training efficiency \\
            $\Bar{\epsilon_e}$ & upper constraint of training efficiency \\
            $F_k$ & inference model of client $k$ \\
            \hline
        \end{tabular}
    \label{tab1}
    \end{center}
\end{table}

% In the federated learning framework, different privacy protection mechanism are applied to prevent privacy leakage during the information exchange process of training. 
% It has been shown that training image within clients can be inverted from the information exchange process \cite{he2019model,zhu2019deep}. To prevent such attacks, there are different categories of protection mechanisms as Randomization \cite{abadi2016deeplearningwithDP,truex2020ldp,geyer2017differentially}, Sparsity \cite{gupta2018distributed,thapa2022splitfed,shokri2015privacy}, and Homomorphic Encryption \cite{gentry2009fully,zhang2020batchcrypt}. 
% The privacy protection mechanism prevents other clients to infer local data, which is known as semi-honest adversaries \cite{zhang2022nofreelunch}.
\section{Related Work}

\subsection{Federated Learning}

With the aim of privacy protection and further improve the model performance, federated learning is proposed, where model is trained on distributed data with different privacy protection mechanisms. Clients update the model locally and periodically communicate with the central server to synchronize the model.
Typically, federated learning deals with a single optimization problem where $K$ clients collaboratively trains the model parameters $w_{fed}$ \cite{kang2023optimizing}.
\begin{equation}
    \begin{split}
        & \mathop{min}\limits_{w_{fed}} \epsilon_u(w_{fed}) \triangleq \sum_{k=1}^{K} \frac{n_k}{n} F_k(w_{fed}) \\
        & F_k(w_{fed})=\mathbb{E}_{\xi \sim D_k}F(w_{fed};\xi) \\
    \end{split}
\end{equation}
$n_k$ represents the size of dataset $D_k$ kept within client $k$ and n is the total size of dataset as $n=\sum_{k=1}^K n_k$. The local model with the $k^{th}$ client $F_k$ is the expectation of the loss function regarding sampling from local dataset $D_k$.

The most well studied algorithm is federated average (FedAVG) and federated stochastic gradient descent 
(FedSGD) \cite{mcmahan2017communicationefficient}. 
FedAVG parallel optimizes the local objective function by using stochastic gradient descent in each client and use trivial average to aggregate the model parameters at the server. 
FedSGD calculates the model update using randomly sampled local data, and upload the model update to central server.
FedAVG and FedSGD is equivalent under the scenario that the number of local epochs $E$ equals one.

\subsection{Differential Privacy}

Differential privacy serves as a strong standard privacy guarantee, which is firstly introduced to protect single instance privacy in terms of adjacent databases \cite{dwork2006calibrating, dwork2010difficulties, dwork2006differential}.
Treating the image-label pair as a instance in database, the $(\epsilon, \delta)$-differential privacy is defined as follows.
\begin{definition}[$(\epsilon, \delta)$-differential privacy in \cite{abadi2016deeplearningwithDP}]
A randomized mechanism $\mathcal{M}: \mathcal{D} \to \mathcal{R}$ with domain $\mathcal{D}$ and range $\mathcal{R}$ satisfies $(\epsilon, \delta)$-differential privacy if for any two adjacent inputs $d$, $d^{\prime} \in \mathcal{D}$ and for any subset of outputs $\mathcal{S} \subseteq \mathcal{R}$ it holds that
$$Pr[\mathcal{M}(d)\in \mathcal{S}] \leq e^{\epsilon}Pr[\mathcal{M}(d^{\prime})\in \mathcal{S}]+\delta$$
where we say that two of these sets are adjacent if they differ in a single entry.
\end{definition}

In deep learning scenario, a series of works under different assumptions are proposed to tighten the privacy leakage bound \cite{bu2020deep, mironov2017renyi} and moments accountant is introduced to quantitatively measure the privacy leakage \cite{abadi2016deeplearningwithDP}.  
\begin{theorem}[Thm. 1 of \cite{abadi2016deeplearningwithDP}]
There exists constant $r$ and constant $s$ so that given the sampling probability $q=\frac{B}{N}$ ($N$ is the number of training set) and the number of total rounds $T$, for any $\epsilon < r q^2 T$, the differentially private SGD algorithm \cite{abadi2016deeplearningwithDP} is $(\epsilon, \delta)$-differential private for any $\delta > 0$ if we choose
$$\sigma \geq s \frac{q \sqrt{T log(1/\delta)}}{\epsilon}$$
\label{theorem:DP}
\end{theorem}

In federated learning, differential privacy also serves as a golden benchmark which people design different algorithms to achieve \cite{seif2020wireless, truex2020ldp}.

% Two ways of training in federated learning with differential privacy is widely used, but both of these algorithms cannot reach Pareto optimal.

% One way is to traverse different $\sigma$ and train until the model convergence \cite{yang2019federatedconceptsandapplications} (Training Until Convergence). 
% The key drawback of this algorithm is that it has to train too many rounds until convergence. Too many rounds lead to only slightly decrease of the utility loss but far more privacy leakage as it exposes the training data too many rounds.
% In other words, with early stopping instead of convergence, we can achieve the same utility with less privacy leakage.

% The other way is to try different $\sigma$ and terminates till using up all the privacy budget with given privacy leakage level \cite{geyer2017differentially} (Training With Budget). With different $\sigma$, the algorithm can achieve different utility with same privacy budget. Under the two-objective optimization formulation, a less privacy leakage can be achieved with the same utility by modifying $\sigma$ and $T$ simultaneously.

\subsection{Multi-objective Optimization}

In multi-objective optimization, the aim is to find a $x$ in the decision space $\mathcal{R}^d$ which can optimize a set of $m$ objective functions as $f_1(w), f_2(w), \dots, f_m(w)$\cite{ehrgott2005multicriteria}.
$$\mathop{min}\limits_{w \in \mathcal{R}^d} G(x) \coloneqq min(f_1(w), f_2(w), \dots, f_m(w))$$

In non-trivial case, all the objective functions cannot achieve their global optimum with the same $x$. The multi-objective optimization methods are used to deal with the contradictions and achieve different optimal trade-off among the objectives. From the point view of decision makers, the multi-objective optimization provides a set of optimal solutions based on different preference and requirements.

\begin{definition}[Pareto dominance in \cite{gunantara2018review}]
For $x_a, x_b \in \mathcal{R}^d$, we say $x_a$ dominates $x_b$ if and only if $f_i(x_a) < f_i(x_b)$ for at least one $i \in [1,2,\dots, n]$ and $f_i(x_a) \leq f_i(x_b)$ for all $i \in [1,2,\dots, n]$.
\end{definition}

\begin{definition}[Pareto optimal solution in \cite{gunantara2018review}]
We say $x^*$ is a Pareto optimal solution if $x^*$ dominates all other $x^{\prime} \in \mathcal{R}^{d}$. 
\end{definition}

\begin{definition}[Pareto set and front in \cite{gunantara2018review}]
Pareto set is the set of $G(x_i)\  i \in [1,2,\dots,n]$, where $x_i \ i \in [1,2,\dots,n]$ is all the Pareto optimal solutions.
Pareto front is the plot of Pareto set $G(x_i)\  i \in [1,2,\dots,n]$ in the objective space.
\end{definition}

Multi-objective optimization can be a challenging job in federated learning.
People use different multi-objective optimization methods as evolutionary algorithms \cite{deb2002fast,deb2014moea3,kim2004spea2+,zhang2007moea}, Bayesian optimization \cite{biswas2022multi,daulton2022multi,laumanns2002bayesian,yang2019multi}, and gradient-based method \cite{desideri2012mgda,desideri2012multiple,liu2021profiling,liu2021stochastic,mahapatra2020multi} to deal with the multi-objective optimization problem in federated learning.
Facing expensive black box scenarios as federated learning, evolutionary algorithms suffers from high computational cost especially facing expensive scenarios as federated learning. 
Bayesian optimization improves the computational efficiency, but its performance highly depends on the surrogate model and acquisition function.
The gradient descent method improves the computational efficiency by finding the directions that simultaneously descend all the objective functions, but requires the gradient information of the objective function. 
\section{MOO in DPFL} \label{Section: MOO in DPFL}
In this section, we formulate the \XY{constrained} bi-objective optimization in Differential Privacy Federated Learning (DPFL) by optimizing the privacy leakage and utility loss simultaneously \XY{with training efficiency constraint}.

\subsection{Setting and Threat Model}
We consider \textit{horizontal federated learning} (HFL) in this paper, which involves $K$ participating parties that each holds a private dataset $D_k, k \in [K]$. We assume the attacker to be \textit{semi-honest}, i.e., it may launch \textit{privacy attacks} on exchanged information to infer participants' private data. For instance, the semi-honest adversary may reconstruct the client's data via the exchanged model gradients \cite{he2019model,zhu2019deep}. 

To mitigate the privacy leakage, each participant applies a protection mechanism to the model information that will be shared with the server. This paper focuses on the differential privacy \cite{abadi2016deeplearningwithDP}, i.e., the local client adds noise on the model gradients. The training procedure involves at least three following steps in each round $t$ (also see Algo. \ref{algorithm: FedSGDwithDP}):
% \vspace{-0.5em}
\begin{enumerate}

\item Each client $k$ trains its local model using its private data set $D_k$ for $E$ local epochs, and obtains the local model $w_k^{t,E}$ as in line 12-16 in Algo. \ref{algorithm: FedSGDwithDP}.

\item In order to prevent semi-honest adversaries from inferring other clients' private information $D_k$, each client $k$ clips model gradients $\Delta w_k^t$ and adds Gaussian noise $n_k^t$ as shown in line 17-18 of Algo. \ref{algorithm: FedSGDwithDP} The client sends protected model gradients $\Tilde{\Delta} w_k^t$ to the server as line 19 of Algo. \ref{algorithm: FedSGDwithDP}. 

\item The server aggregates $\Tilde{\Delta} w_k^t, k=1,\cdots,K$ by average in line 8 of Algo. \ref{algorithm: FedSGDwithDP}. The global model $w^{t+1}$ is updated to be $w^{t+1}=w^{t}+ \frac{1}{K}\sum_{i=1}^{K}\Tilde{\Delta} w_k^t$. $w^{t+1} \gets w^{t}+ \frac{1}{|P_t|}\sum_{k\in |P_t|}\Tilde{\Delta} w_k^t$. Then the server distributes global model $w^{t+1}$ to all clients.

\end{enumerate}

% The processes \circled{1}-\circled{4} iterate until the utility of the aggregated model $W_{\text{fed}}^{D}$ does not improve.\\

\begin{algorithm}
    \caption{
    \textbf{DP-FedSGD}$(T,q,\sigma)$ (Algo. 3 of \cite{zhang2022understanding}): 
    The $K$ clients are indexed by $k$, $P_t$ is the set of participating clients in round $t$; %$|P_t|$ is the number of participating clients in round $t$
    $B$ is the local batch size, $T$ is communication rounds, $E$ is local training epochs;
    $w^t$ is the global model at round $t$, $w_k^{t,e}$ is the local model of client $k$ at round $t$ and local epoch $e$, $\eta$ is the learning rate, $c_{clip}$ is the clipping constant, $\sigma$ is noise level, $n_k^t$ follows $\mathcal{N}(0, \sigma^2)$; $D_{test}$ is the test dataset, $(x_i,y_i)$ is the feature-label pair in $D_{test}$.}
    \begin{algorithmic}[1]
        \STATE \textbf{Server executes:}
            \STATE Randomly initialize $w^0$
            \FOR{$t \in [0,1,\dots,T-1]$}
                \STATE Distribute $w_t$ to clients
                \FOR{Client $k \in [P_t]$ in parallel} 
                    \STATE $\Tilde{\Delta} w_k^t \gets$ \textbf{ClientUpdate}($k$, $w^{t}$)
                \ENDFOR
            \STATE Aggregate model $w^{t+1} \gets w^{t}+ \frac{1}{|P_t|}\sum_{k\in |P_t|}\Tilde{\Delta} w_k^t$
            % \STATE Evaluate test loss $L_{t} = - \sum_{i=1}^{|D_{test}|}y_i log(F(x_i;w^t))$, $(x_i, y_i) \in D_{test}$; \HL{wrong}
            \STATE Evaluate test loss $L_{ce}^t$ by cross-entropy loss function $L_{ce}(x_i, y_i;w^t)$, $(x_i, y_i) \in D_{test}$ %\XY{giving exact formulation need too many new definitions}
            \ENDFOR
            \STATE Return test loss $L_{ce}^t$ for all $t$, $t \in [1,2,\dots, T]$
        \STATE  
        \STATE \textbf{ClientUpdate}($k$, $w^{t}$):     // Run on client k
        \STATE Initialize: $w_k^{t,0} \gets w^{t}$
            \FOR{$e \in [0,1,\dots,E-1]$}
                \STATE Randomly sample batch $b$ (size $B$) in local training dataset $D_k$;
                \STATE Compute gradient $g_k^{t,e} = \nabla F_k(w_k^{t,e};b)$
                \STATE Local update $w_k^{t,e+1} \gets w_k^{t,e} - \eta g_k^{t,e}$
            \ENDFOR
            \STATE Compute the difference of model weights $\Delta w_k^t = w_k^{t,E}-w_k^{t,0}$
            \STATE Add noise $\Tilde{\Delta} w_k^t = \Delta w_k^t/max(1,\frac{{\vert| \Delta w_k^t \vert|}^2}{c_{clip}}) + n_k^t$
            \STATE Return and upload $\Tilde{\Delta} w_k^t$ to server.
    \end{algorithmic}
\label{algorithm: FedSGDwithDP}
\end{algorithm}

\subsection{Privacy Leakage, Utility Loss, \XY{and Training Efficiency} in DPFL} \label{sec:Privacy Leakage and Utility Loss in DPFL}

In this paper, we consider two objectives in the DPFL, i.e., the privacy leakage and utility loss, which is defined as following.

\noindent\textbf{Privacy Leakage.}
We follow Thm. 1 of \cite{abadi2016deeplearningwithDP} and Thm. 3.2 of  \cite{zhang2022understanding} to provide the definition of the differential privacy $\epsilon_p$ leakage of local client in federated learning as:
\begin{equation} \label{eq:privacy-budget}
    \epsilon_p = C \frac{c_{clip} \sqrt{qTlog(1/\delta)}}{\sqrt{K} \sigma},
\end{equation}
where $C$ is a constant, $c_{clip}$ is the clipping constant in differential privacy, $K$ is the total number of participants in federated learning and $q$ is the the sample ratio representing the fraction of participating clients among all clients in each round.

According to Eq. \eqref{eq:privacy-budget}, the privacy budget is influenced by three factors: $q$, $T$, and $\sigma$. Specifically, the privacy leakage is positively related to $q$ and $T$ while negatively related to $\sigma$.

\noindent\textbf{Utility Loss.}
The utility loss $\epsilon_u$ of a federated learning system is defined as follows:
\begin{equation}\label{def:utility}
        \epsilon_u = \text{U}({w_{\text{fed}}^{O}}) - \text{U}(w_{\text{fed}}^{D}),
    \end{equation}
% \end{definition}
where $\text{U}(w_{\text{fed}}^{D})$ and $\text{U}(w_{\text{fed}}^{O})$ measure the utility of protected global model $w_{\text{fed}}^{D}$ and unprotected global model $w_{\text{fed}}^{O}$, respectively. 
Moreover, the following Lemma \ref{Lemma: Upper Bound for DP-FedAvg} (Cor. 3.2.1 of \cite{zhang2022understanding}) provides the theoretical upper bound for Algo. \ref{algorithm: FedSGDwithDP}.

\begin{lem}[Adapted from Cor. 3.2.1 of \cite{zhang2022understanding}]
% (Cor. 3.2.1 of \cite{zhang2022understanding}).
For Algo. \ref{algorithm: FedSGDwithDP}, 
assume $F_k(x)$ satisfies $\vert| \nabla F_k(x) - \nabla F_k(y) \vert | \leq L \vert| x-y \vert |, \forall k,x,y$, $\min_xF(x) \geq F^{*}$,
% ; $\mathbb{E}[\vert| g_k^{t,e}-\nabla F_k(x_i^{t,e})\vert|^2] \leq \sigma_l^2$, $\vert| \nabla F_k(x) - \nabla F(x) \vert |^2 \leq \sigma_g^2$, 
% $\vert| g_k^{t,e} \vert| \leq G, \forall t,e,k$, where $L$ is the Lipschitz constant of gradient, $\sigma_l^2$ and $\sigma_g^2$ are intra-client and inter-client gradient variance, 
$G$ is the bound on stochastic gradient, $C_{clip}$ is the clipping constant with $C_{clip} \geq \eta EG$.
$\frac{\eta}{qK} \leq \eta$ as $qK \geq 1$.

\noindent By letting $\eta \leq min\{\frac{qK}{6EL(P-1)}, \frac{qK}{96E^2}, \frac{1}{\sqrt{60}EL}\} $, we have 
\begin{equation*}
\frac{1}{T}\sum_{t=1}^T\mathbb{E} [\vert| \nabla F(x^t) \vert|^2] \leq \mathbf{O} (\frac{1}{\eta E T} + \eta^2E^2+\eta) + \mathbf{O} (\frac{\sigma^2}{\eta qKE}),\end{equation*}
where $T$ and $E$ are communication rounds and local training epochs, $\eta$ is learning rate, $K$ is the total number of clients, $q$ is sample ratio, and $\sigma$ is the noise level in differential privacy (standard derivation of added noise).
    \label{Lemma: Upper Bound for DP-FedAvg}
\end{lem}

% Moreover, considering $\frac{\eta}{qK} \leq \frac{\eta}{K}$ and $\frac{\eta}{K}$ tends to be small, we can obtain
% \begin{equation*}
% \frac{1}{T}\sum_{t=1}^T\mathbb{E} [\vert| \nabla f(x^t) \vert|^2] \leq \mathbf{O} (\frac{1}{\eta E T} + \eta^2E^2) + \mathbf{O} (\frac{\sigma^2}{\eta qKE}) + k_0
% \end{equation*} 
% where $k_0$ is a small constant.
Note that the upper bound of the utility loss goes up with the increase of $\sigma$ and the decrease of $q$, $T$.

\begin{remark}
$K$ and $E$ are always pre-defined, which is explored in the experimental ablation study part \ref{sec: ablation study}. $c_{clip}$ is a given constant.
\end{remark}

\XY{
\noindent\textbf{Training Efficiency.}
We use the system training time to describe the training efficiency in DPFL as follows:
\begin{equation}
    \epsilon_e = c_t T
    \label{eq:training_efficiency}
\end{equation}
% \end{definition}
where $T$ is the communication rounds and $c_t$ is the per-round training time in DPFL, which is treated as a constant due to the nearly uniform training time per round.
}

\subsection{\XY{Constrained} Bi-Objective Optimization in DPFL}\label{sec: Bi-Objective Optimization in DPFL}
Conventionally, existing work aims to minimize the utility loss given the privacy budget $\epsilon_0$ \cite{geyer2017differentially,wei2020federated} , which can be formulated as:
\begin{equation}
    \begin{split}
        &\min\limits_{T, \sigma} \epsilon_u(T, \sigma),  \text{ where } \epsilon_u(T, \sigma) = \sum_{k=1}^Kp_{k}F_{k}(T, \sigma) \\
        & \text{  subject to  } \,\, \epsilon_p(T, \sigma) \leq \Bar\epsilon_0
    \end{split}
    \label{equation: convolutional formulation}
\end{equation}
\HL{where $\epsilon_p$, $\epsilon_u$ and $\Bar{\epsilon_0}$ represent privacy leakage, utility loss and the upper constraint of privacy leakage respectively.}

\XY{However, the existing work only focuses on utility loss and privacy leakage, but ignore the training efficiency. It means that existing work cannot ensure the acceptable training time in DPFL.} 
\XY{Moreover,} both the privacy leakage and utility loss are influenced by the training epoch ($T$), noise level ($\sigma$) and sample ratio ($q$) illustrated in the former section.
It means Eq. \eqref{equation: convolutional formulation} is insufficient to obtain a complete Pareto optimal as it only considers the influence of $\sigma$ and $T$. 
\XY{In this work, we reformulate the optimization problem as the following definition.}

\begin{definition}
    \Yang{
    The utility loss and privacy leakage bi-objective optimization problem with training efficiency constraint w.r.t. noise level ($\sigma$), communication round ($T$), and sample ratio ($q$) in DPFL is:
    \begin{equation}
        \begin{aligned}
            &\min\limits_{T, \sigma, q} ( \epsilon_u(T, \sigma, q), \epsilon_p(T, \sigma, q)),  \\
            & \text{ where } \epsilon_u(T, \sigma, q) = \sum_{k=1}^Kp_{k}F_{k}(T, \sigma, q)\\
            &  \ \ \ \ \ \ \ \ \  \epsilon_p(T, \sigma, q) = C \frac{c_{clip} \sqrt{qTlog(1/\delta)}}{\sqrt{K} \sigma}\\
            & \text{  subject to  } \,\, \XY{\epsilon_e(T,\sigma,q) \leq \Bar{\epsilon_e}},  \\
            % & \ \ \ \ \ \ \ \ \ \ \ \ \ \ \XY{0 < \sigma \leq \sigma_{max}, \ 0 < q \leq q_{max}},
        \end{aligned}
        \label{equation: formulation_3factor}
    \end{equation}
    where $\epsilon_u$ represents utility loss, $\epsilon_p$ represents the privacy leakage by Eq. \eqref{eq:privacy-budget}, $\epsilon_e$ represents the training efficiency by Eq. \eqref{eq:training_efficiency}, $p_{k}$ is the coefficient of $F_{k}$ satisfying $\sum_{k=1}^Kp_{k}=1$, 
    $\Bar{\epsilon_e}$ is the upper constraint of training efficiency.
    % $\sigma_{max}$ is the upper constraint of noise level ($\sigma$), and $q_{max}$ is the upper constraint of sample ratio ($q$)}. 
    }
\end{definition}

\HL{In the following sections, we provide the analysis the constrained bi-objective optimization problem both theoretically and experimentally.}
On one hand, we provide the theoretical analysis on the Pareto optimal solutions of Eq. \eqref{equation: formulation_3factor} in Sec. \ref{Section: theoretical analyzing}.
On the other hand, we solve the Eq. \eqref{equation: formulation_3factor} by non-dominated sorting Algo. \ref{algorithm: non-dominated sorting} in Sect. \ref{sec: 3factorexperiment} to determine the Pareto optimal solutions.  Moreover, the experimental results in Sect. \ref{Section: experiment} can validate the theoretical conclusion.

\begin{algorithm}
    \caption{Multi-Objective Optimization in DPFL. }
    \begin{algorithmic}
        \STATE \textbf{Step1 Objective Function Calculation:}
        \FOR{sample ratio $q_i \in (0,1]$}
            \FOR{$\sigma_i \in (0,\sigma_{max}]$}
                % \STATE Run Algo. \ref{algorithm: FedSGDwithDP} DP-FedGD($T = T_{max}, \sigma_i, q_i$)
                \STATE Calculate test loss $L_t$ for $\forall T \in [T_{max}]$ by Algo. \ref{algorithm: FedSGDwithDP} as $(L_1, L_2, \dots, L_{T_{max}}) = \text{DP-FedSGD}(T_{max}, \sigma_i, q_i)$
                \STATE Calculate privacy leakage by Eq. \eqref{eq:privacy-budget} as $\epsilon_p(T, \sigma_i, q_i) = C \frac{c_{clip} \sqrt{q_i tlog(1/\delta)}}{\sqrt{K}\sigma_i}\  \forall T \in [T_{max}]$ 
            \ENDFOR
        \ENDFOR
        \STATE 
        \STATE \textbf{Step2 Non-dominated Comparison:}
        \FOR{$(T_i,\sigma_i, q_i)$, $\forall T_i, \sigma_i, q_i$}
            \IF{$(T_i,\sigma_i, q_i)$ Pareto dominates all $(T_j,\sigma_j, q_j), \forall T_j, \sigma_j, q_j$}
            \STATE $(T_i,\sigma_i, q_i)$ is one of the Pareto optimal solutions
            \ENDIF
        \ENDFOR
        \STATE
        \STATE \textbf{Step3 Output:}
        \STATE Return the Pareto optimal solutions
    \end{algorithmic}
    \label{algorithm: non-dominated sorting}
\end{algorithm}
\section{Theoretical Analysis} \label{Section: theoretical analyzing}

In this section, we commence by simplifying the \XY{constrained} bi-objective optimization problem through the utilization of an upper boundary for the utility loss in Sect. \ref{Sec: Simplified Formulation}. 
Subsequently, we conduct a comprehensive theoretical investigation into the \XY{constrained} bi-objective problem, culminating in the derivation of an analytical expression for the Pareto solution involving three distinct parameters: communication rounds ($T$), noise level ($\sigma$), and sample ratio ($q$) in Sect. \ref{sec: analytical solution}. 
Our approach entails initially analysing the Pareto front and Pareto solution encompassing all three parameters \XY{with unconstrained $q$ and $\sigma$}. 
Moreover, with given $q$ \XY{and constrained $\sigma$}, we focus on $T$ and $\sigma$ and detailed analyze the Pareto solution in different cases.

\subsection{Simplified Formulation}\label{Sec: Simplified Formulation}
We obtain a simplified version of the \XY{constrained} bi-objective optimization formulation by using the differential privacy leakage $\epsilon_p(T, \sigma, q)$, upper bound of utility loss $\epsilon_u(T, \sigma, q)$, \XY{and training efficiency $\epsilon_e$} illustrated in Sect. \ref{sec:Privacy Leakage and Utility Loss in DPFL}.

\XY{
\begin{equation}
    \begin{aligned}
        & \text{min}(f_1(T,\sigma, q), f_2(T,\sigma, q)) \\
        & \text{where } f_1(T,\sigma,q) = \mathbf{O} (\frac{1}{\eta E T} + \eta^2E^2 + \eta) + \mathbf{O} (\frac{\sigma^2}{\eta qKE}) \\
        & \ \ \ \ \ \ \ \ f_2(T, \sigma, q) = C \frac{c_{clip} \sqrt{qTlog(1/\delta)}}{\sqrt{K} \sigma} \\
        & \text{subject to } f_3(T,\sigma,q) = c_t T \leq \Bar{\epsilon_e} \\
    \end{aligned}
    \label{equation: upperbound_MOO}
\end{equation}
}

As the number of clients ($K$), the local training epochs ($E$), the learning rate ($\eta$) and the clipping constant ($c_{clip}$) are usually pre-decided in the real scenario, 
\XY{and the per-round training time $c_t$ is approximately uniform each round}, 
we keep $K$, $E$, $c_{clip}$, and \XY{$c_t$} as constants. 

Focusing on parameter $\sigma$, $T$ and $q$, the \XY{constrained} bi-objective formulation can be further simplified as follows according to Lemma \ref{lemma:MOOequivalence}.
\XY{
\begin{equation}
    \begin{aligned}
        & \text{min}(f_1(T,\sigma, q), f_2(T,\sigma, q)) \\
        & \text{where } f_1(T, \sigma, q) = \frac{1}{T} + k\frac{\sigma^2}{q K},\ k \text{ is constant} \\
        & \ \ \ \ \ \ \ \ f_2(T, \sigma, q) = \frac{\sqrt{qT}}{\sigma} \\
        & \text{subject to } f_3(T,\sigma,q) = c_t T \leq \Bar{\epsilon_e} \\
    \end{aligned}
    \label{equation: simplifiedMOO}
\end{equation}
}

\XY{With aim of simplifying the optimization objective functions, we provide the following lemma.}
\begin{lem}
    The Pareto optimal solutions of Eq. \eqref{equation: original} and Eq. \eqref{equation: simplified} are equivalent, where constants $a_1, a_2 \in \mathbf{R}^{+}$ and constants $m_1, m_2 \in \mathbf{R}$.
    \begin{equation}
        \left\{
        \begin{aligned}
            & f_1(x) = a_1 \times h_1(x) + m_1\\
            & f_2(x) = a_2 \times h_2(x) + m_2\\
        \end{aligned}
        \right.
        \label{equation: original}
    \end{equation}

    \begin{equation}
        \left\{
        \begin{aligned}
            & f_1^{\prime}(x) = h_1(x)\\
            & f_2^{\prime}(x) = h_2(x)\\
        \end{aligned}
        \right.
        \label{equation: simplified}
    \end{equation}
    \label{lemma:MOOequivalence}
\end{lem}

\begin{proof} 
Suppose that $x_0$ is a Pareto optimal solution of Eq. \eqref{equation: original}. We have Eq. \eqref{eq: paretosolution-definition} according to the definition of Pareto optimal solution.
    \begin{equation}
        \begin{split}
            \forall x_i \neq x_0,\  &  \forall\ j \in [1,2],\  f_j(x_0) \leq f_j(x_i)\\
            & \exists\ j \in [1,2],\  f_j(x_0) < f_j(x_i)
        \end{split}
        \label{eq: paretosolution-definition}
    \end{equation}
    \begin{equation}
        \begin{split}
            \iff & \forall x_i \neq x_0, \\
            & \forall\ j \in [1,2],\  a_j h_j(x_0) + m_j \leq a_j h_j(x_i) + m_j\\
            & \exists\ j \in [1,2],\  a_j h_j(x_0) + m_j < a_j h_j(x_i) + m_j \\
        \end{split}
    \end{equation}
    \begin{equation}
        \begin{split}
            \iff \forall x_i \neq x_0,\  &  \forall\ j \in [1,2],\  h_j(x_0) \leq h_j(x_i)\\
            & \exists\ j \in [1,2],\  h_j(x_0) < h_j(x_i)
        \end{split}
    \end{equation}
    \begin{equation}
        \begin{split}
            \iff \forall x_i \neq x_0,\  &  \forall\ j \in [1,2],\  f_j^{\prime}(x_0) \leq f_j^{\prime}(x_i)\\
            & \exists\ j \in [1,2],\  f_j^{\prime}(x_0) < f_j^{\prime}(x_i)
        \end{split}
    \end{equation}
    It has been proved that $x_0$ is still a Pareto optimal solution of optimization problem Eq. \eqref{equation: simplified}. The Pareto optimal solutions of Eq. \eqref{equation: original} and Eq. \eqref{equation: simplified} are equivalent.
\end{proof}

\subsection{Pareto Optimal Solutions} \label{sec: analytical solution}
In this section, we derive the analytical Pareto optimal solutions w.r.t $T, \sigma$ and $q$ in Thm. \ref{Thm: analytical solution} and the the analytical Pareto optimal solutions with pre-defined $q$ under different constraint cases in Cor. \ref{Cor: analytical solution with given q}.
\begin{theorem}
\XY{\textbf{(Analytical Solutions w.r.t.} $\mathbf{q, T, \sigma}$  \textbf{with unconstrained $\mathbf{\sigma}$ and $\mathbf{q}$)}}
Assuming $\sigma \in [0,+\infty) \text{and } q \in (0,1]$, we have the Pareto optimal solutions for Eq. \eqref{equation: simplifiedMOO} as follows.
    \begin{equation}
        k \sigma^2 T = q K,
        \label{eq: noconstraintsolution}
    \end{equation}
\XY{where $k, t_c$ is constant and $T \in [1, \dots, \lfloor\frac{\Bar{\epsilon_r}}{t_c}\rfloor]$}.
  \label{Thm: analytical solution}
\end{theorem}

\begin{proof}\label{proof: analytical solution}
    Set $X = \frac{1}{T} + \frac{k}{K}\frac{\sigma^2}{q}$ with constant $k$. The \XY{bi-objective optimization objective functions $f_1$ and $f_2$ in Eq. \eqref{equation: simplifiedMOO} is converted to the following form:}
    \begin{equation}
        \begin{aligned}
            & f_1(X, \sigma, q) = X, \\
            & f_2(X, \sigma, q) = \sqrt{\frac{q}{\sigma^2(X-k \frac{\sigma^2}{qK})}}, \\
        \end{aligned}
        \label{eq:X}
    \end{equation}
    \XY{where $X \in (0, +\infty)$, $q \in (0,+\infty)$, $\sigma \in (0,+\infty)$.}
 
    For a specific $X$, $f_2$ reaches minimum $\frac{2\sqrt{k}}{\sqrt{K}X}$ when $\frac{\sigma^2}{q} = \frac{KX}{2k}$ according to inequality of arithmetic means. Specifically, 
    \begin{equation} 
        \begin{aligned}
        f_2(\sigma^2, q|X) & =\sqrt{\frac{1}{\frac{\sigma^2}{q}(X-\frac{k}{K}\frac{\sigma^2}{q})}} \\
        & \geq f_2(\frac{\sigma^2}{q}=\frac{K X}{2k}|X) = \frac{2\sqrt{k}}{\sqrt{K}X}. \\
        \end{aligned}
    \end{equation} 
    
    Therefore, 
    given $X$, $(X, \frac{2\sqrt{k}}{\sqrt{K}X})$ Pareto dominates the set $\{(X, \sqrt{\frac{q}{\sigma^2(X-k\sigma^2/(qK)}})| \frac{\sigma^2}{q} \neq \frac{KX}{2k}, \forall \sigma,q \}$. 
    As a result, the Pareto set of Eq. \eqref{equation: simplifiedMOO} must be a subset of $\{(X, \frac{2\sqrt{k}}{\sqrt{K}X}), \forall X \} \triangleq S$.
    Since each points in $S$ non-dominates each other, the set $S$ is the Pareto set of Eq. \eqref{equation: simplifiedMOO}.

    Moreover, by substituting $X$ via Eq. $\frac{\sigma^2}{q} = \frac{K X}{2k}$, we have proven that the Pareto optimal solutions follows:
    \begin{equation}
        \begin{aligned}
            \frac{\sigma^2}{q} = \frac{KX}{2k} & \iff \frac{\sigma^2}{q} = \frac{K(\frac{1}{T}+k\frac{\sigma^2}{qK})}{2k} \\
            & \iff k \sigma^2 T= q K. \\
        \end{aligned}
        \label{eq:optimalrelationship}
    \end{equation}
    \XY{where $k, t_c$ is constant and $T \in [1, \dots, \lfloor\frac{\Bar{\epsilon_r}}{t_c}\rfloor]$}.

    We have finished the proof. 
    \label{proof:analytical solution}
\end{proof}

\begin{figure*}[ht]
    \centering
    \subfigure[\XY{Case-I ($\sigma \in [0,+\infty)$)}]{
        \includegraphics[scale=0.5]{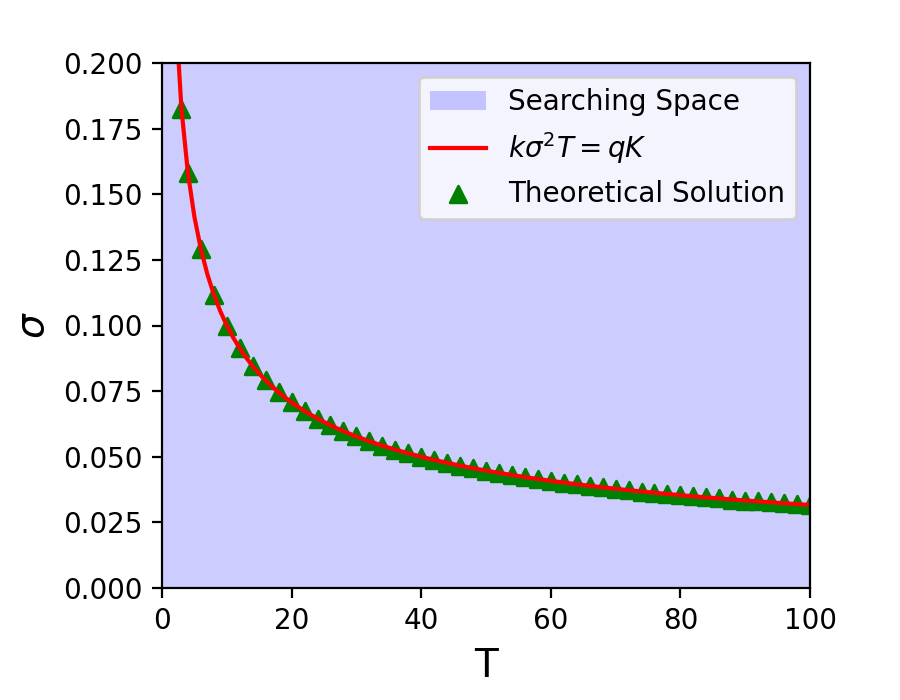} 
        \label{Fig:Analytical Solution-A}
    }
    \subfigure[\XY{Case-II (Constrained $\sigma$ with $k \sigma_{max}^2 \lfloor\frac{\Bar{\epsilon_r}}{t_c}\rfloor > {qK}$)}]{
        \includegraphics[scale=0.5]{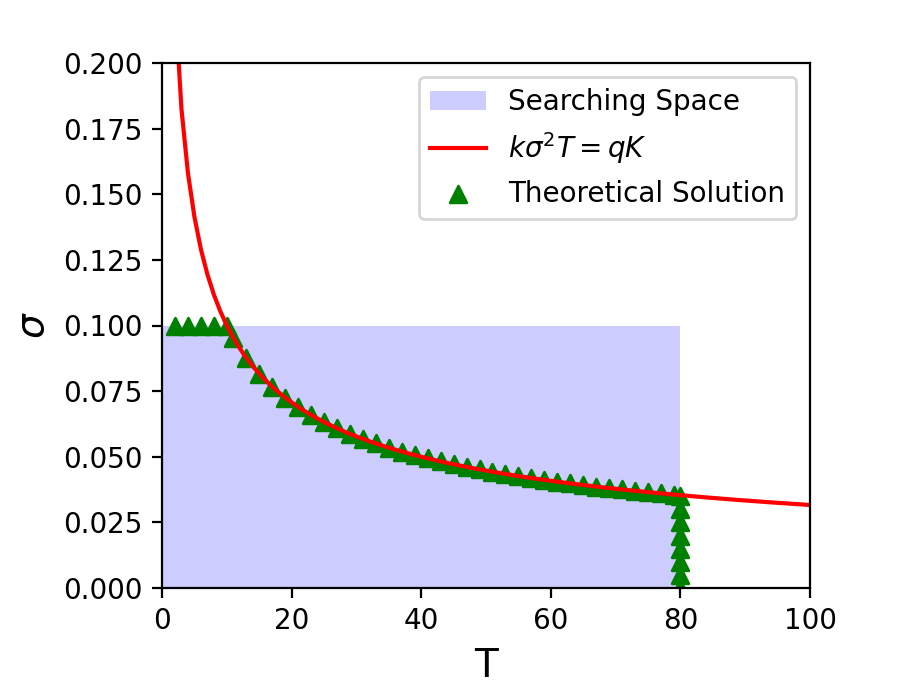} 
        \label{Fig:Analytical Solution-B}
    }
    \subfigure[\XY{Case-III (Constrained $\sigma$ with $k \sigma_{max}^2 \lfloor\frac{\Bar{\epsilon_r}}{t_c}\rfloor \leq qK$)}]{
        \includegraphics[scale=0.5]{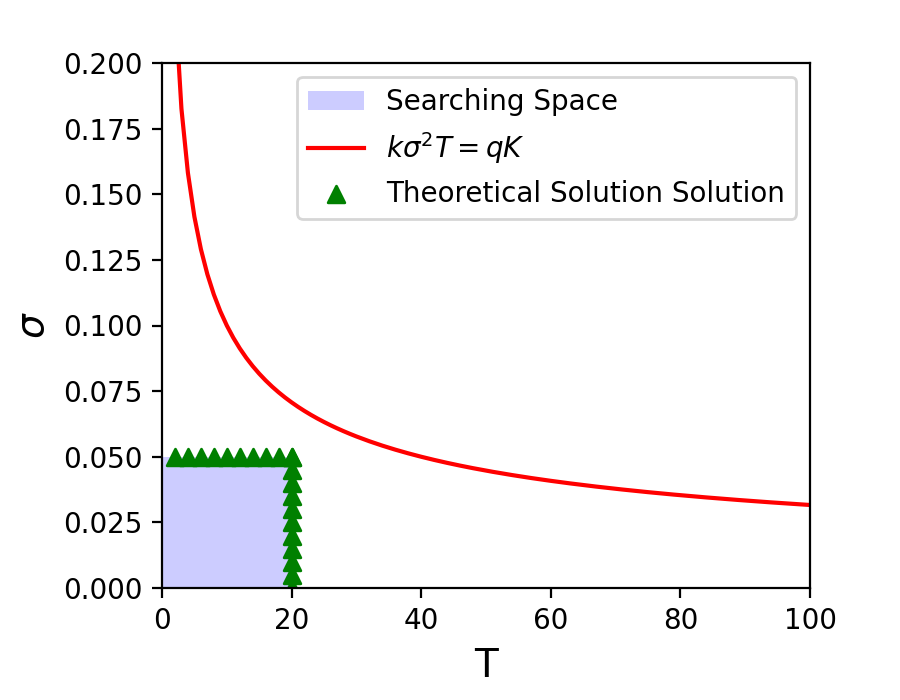}
        \label{Fig:Analytical Solution-C}
    }
    \DeclareGraphicsExtensions.
    \caption{\XY{Theoretical Pareto Optimal Solutions}. $T$ is on x-axis and $\sigma$ is on y-axis. The blue area is the decision space with constraint. The red line represents the inverse proportional relationship between $\sigma^2$ and $T$, and the green triangle represents the analytical solution.}
    \label{Fig:Analytical Solution}
\end{figure*}

\begin{remark}\label{remark}
    In Proof \ref{proof: analytical solution}, the \XY{bi-objective optimization objective functions} as Eq. \eqref{eq:X} can be written w.r.t. $\frac{\sigma^2}{q}$ and $T$ as follows:
    \begin{equation}
        \begin{aligned}
            & f_1(X, \frac{\sigma^2}{q}) = X,\\
            & f_2(X, \frac{\sigma^2}{q}) = \sqrt{\frac{1}{\frac{\sigma^2}{q}(X-\frac{k}{K} \frac{\sigma^2}{q})}}.\\
        \end{aligned}
    \end{equation}
    where $X = \frac{1}{T} + \frac{k}{K}\frac{\sigma^2}{q}$.
    
    The Pareto optimal solutions as Eq. \eqref{eq:optimalrelationship} can also be written w.r.t. $\frac{\sigma^2}{q}$ and $T$ as follows:
    \begin{equation}
        \frac{\sigma^2}{q} T=\frac{K}{k}.
    \end{equation}
    
    Therefore, the fraction $\frac{\sigma^2}{q}$ can be regraded as a single parameter in Proof \ref{proof:analytical solution}. Based on that, keeping one of $\sigma$, $q$ as a constant and iterating over the other one can still achieve the whole Pareto front \XY{with unconstrained $\sigma$ and $q$}. 
    % with different constraints as no constraint case, wide range case (k$\sigma_{max}^2 T_{max} > {qK}$), and small range case ($k\sigma_{max}^2 T_{max} \leq {qK}$) shown in Fig. \ref{Fig:Analytical Solution} respectively.
\end{remark}

In the real scenario, the sample ratio ($q$) is usually decided by the server and distributed to the clients. Based on this common setting, we keep sample ratio ($q$) as a constant and further provides the analytical solutions w.r.t. $\sigma$ and $T$ as Theorem \ref{Cor: analytical solution with given q} under different cases. We also demonstrate the different cases analytical solutions in Fig. \ref{Fig:Analytical Solution}. 
 % Based on this common setting, we keep sample ratio ($q$) as a constant and further provides the analytical solutions w.r.t. $\sigma$ and $T$ with different constraints as no constraint case, wide range case (k$\sigma_{max}^2 T_{max} > {qK}$), and small range case ($k\sigma_{max}^2 T_{max} \leq {qK}$) shown in Fig. \ref{Fig:Analytical Solution} respectively.
% Based on the range of $T$ and $\sigma$, we derive the Pareto optimal solutions in three cases as no constraint case, wide range case (k$\sigma_{max}^2 T_{max} > {qK}$), and small range case ($k\sigma_{max}^2 T_{max} \leq {qK}$) shown in Fig. \ref{Fig:Analytical Solution} respectively.

\begin{theorem} 
        With a pre-defined sample ratio ($q$), the Pareto solution for Eq. \eqref{equation: simplifiedMOO} is as follows in the three cases:

    \begin{itemize}
        \item \XY{Case-I (Unconstrained $\sigma$)}. Assuming $\sigma \in [0,+\infty)$, we have the Pareto optimal solution as follows:
            \begin{equation}
                k \sigma^2 T = {qK},\ k \text{ is constant},
            \end{equation}
            \XY{where $k, t_c, q$ is constant and $T \in [1, \dots, \lfloor\frac{\Bar{\epsilon_r}}{t_c}\rfloor]$}.

        \item \XY{Case-II (Constrained $\sigma$ with $k \sigma_{max}^2 \lfloor\frac{\Bar{\epsilon_r}}{t_c}\rfloor > {qK}$)}. \XY{Assuming $\sigma \in [0,\sigma_{max}]$ and $k \sigma_{max}^2 \lfloor\frac{\Bar{\epsilon_r}}{t_c}\rfloor > {qK}$}, we have the Pareto optimal solutions as follows:
            \begin{equation}
                \left\{
                \begin{aligned}
                    & \sigma = \sigma_{max}\ \ when\  T \in \{1,\dots,\lfloor \frac{qK}{k\sigma_{max}^2} \rfloor\}\\
                    & \sigma = \sqrt{\frac{qK}{kT}}\ \ when\ T \in \{\lceil \frac{qK}{k\sigma_{max}^2} \rceil, \dots, \lfloor\frac{\Bar{\epsilon_r}}{t_c}\rfloor-1\}\\
                    & \sigma = [0,\sqrt{\frac{qK}{k \lfloor\frac{\Bar{\epsilon_r}}{t_c}\rfloor}}]\ \ when\ T = \lfloor\frac{\Bar{\epsilon_r}}{t_c}\rfloor. \\
                \end{aligned}
                \right.
                \label{widerangesolution}
            \end{equation}

        \item \XY{Case-III (Constrained $\sigma$ with $k \sigma_{max}^2 \lfloor\frac{\Bar{\epsilon_r}}{t_c}\rfloor \leq qK$)}. Assuming $\sigma \in [0,\sigma_{max}]$ and $k \sigma_{max}^2 \lfloor\frac{\Bar{\epsilon_r}}{t_c}\rfloor \leq qK$, the Pareto optimal solutions are as follows.
            \begin{equation}
                \left\{
                \begin{aligned}
                    & \sigma = \sigma_{max}\ \ when\  T \in \{1,2,\dots,\lfloor\frac{\Bar{\epsilon_r}}{t_c}\rfloor-1\} \\
                    & \sigma = [0,\sigma_{max}]\ \ when\ T = \lfloor\frac{\Bar{\epsilon_r}}{t_c}\rfloor \\
                \end{aligned}
                \right.
                \label{smallrangesolution}
            \end{equation}
    \end{itemize}
    \label{Cor: analytical solution with given q}
\end{theorem}

\begin{proof}
    Set $X = \frac{1}{T} + \frac{k}{K}\frac{\sigma^2}{q}$ with constant $k$ and $q$. The \XY{bi-objective optimization objective functions $f_1$ and $f_2$ in Eq. \eqref{equation: simplifiedMOO} is converted to the following form Eq. \eqref{equation: changevariable} and Eq. \eqref{equation: changevariable2}:}
    
    \begin{equation}
        \begin{aligned}
            & f_1(X, \sigma) = X\\
            & f_2(X, \sigma) = \sqrt{\frac{q}{\sigma^2(X-k\frac{\sigma^2}{qK})}}\\
        \end{aligned}
        \label{equation: changevariable}
    \end{equation}
    
    \begin{equation}
        \begin{aligned}
            & f_1(X, T) = X\\
            & f_2(X, T) = \sqrt{\frac{kT^2}{K(XT-1)}}\\
        \end{aligned}
        \label{equation: changevariable2}
    \end{equation}

    \XY{Set $T_{max}=\lfloor\frac{\Bar{\epsilon_r}}{t_c}\rfloor$. The constraint as $f_3(T,\sigma,q) = c_t T \leq \Bar{\epsilon_e}$ in Eq. \ref{equation: simplifiedMOO} can be converted to the form:
    \begin{equation}
        \begin{aligned}
            f_3(T,\sigma,q) & = c_t T \leq \Bar{\epsilon_e}, T \in \mathcal{Z}^{+} \iff T \leq \frac{\Bar{\epsilon_e}}{c_t}, T \in \mathcal{Z}^{+} \\
            & \iff T \in [1,2,\dots,T_{max}]. \\
        \end{aligned}
    \end{equation}
    }

    \textbf{Case-I (Unconstrained $\mathbf{\sigma}$)}
    \XY{In Case-I, we have $\sigma \in [0,+\infty)$, given $q$, and $T \in [1,2,\dots, T_{max}]$.} For a specific $X$, $f_2$ reaches minimum value $\frac{2\sqrt{k}}{\sqrt{K}X}$ when $\sigma^2 = \frac{qKX}{2k}$ according to inequality of arithmetic means. Specifically,
    \begin{equation} 
        \begin{aligned}
        f_2(\sigma^2|X) & =\sqrt{\frac{q}{\sigma^2(X-k \frac{\sigma^2}{qK})}} \\
        & \geq f_2(\sigma^2=\frac{q K X}{2k}|X) = \frac{2\sqrt{k}}{\sqrt{K}X}. \\
        \end{aligned}
    \end{equation} 
    
    Therefore, given $X$, $(X, \frac{2\sqrt{k}}{\sqrt{K}X})$ Pareto dominates the set $\{(X, \sqrt{\frac{q}{\sigma^2(X-k\sigma^2/(qK))}})| \sigma^2 \neq \frac{qKX}{2k}, \forall \sigma \}$.
    As a result, the Pareto set of Eq. \eqref{equation: simplifiedMOO} must be a subset of $\{(X, \frac{2\sqrt{k}}{\sqrt{K}X}), \forall X \} \triangleq S$.
    Since each points in $S$ non-dominates each other, the set $S$ is the Pareto set of Eq. \eqref{equation: simplifiedMOO}.

    Moreover, by substituting $X$ via Eq. $\sigma^2 = \frac{q K X}{2k}$, we have proven that the Pareto solution follows:
    \begin{equation}
        \begin{aligned}
            \sigma^2 = \frac{qKX}{2k} & \iff \sigma^2 = \frac{qK(\frac{1}{T}+k\frac{\sigma^2}{qK})}{2k} \\
            & \iff k \sigma^2 T= q K. \\
        \end{aligned}
    \end{equation}
    
     \textbf{Case-II (Constrained $\mathbf{\sigma}$ with $\mathbf{k \sigma_{max}^2 \lfloor\frac{\Bar{\epsilon_r}}{t_c}\rfloor > {qK}}$)}
     \XY{As $T_{max}=\lfloor\frac{\Bar{\epsilon_r}}{t_c}\rfloor$, we have $T \in [1,\dots,T_{max}]$ and $k \sigma_{max}^2 T_{max} > {qK}$ with $\sigma \in [0,\sigma_{max}]$ and given $q$}. We separate the value of $X$ into three parts.
    
    \begin{itemize}
        \item Given $X \in [\frac{1}{T_{max}},\frac{2}{T_{max}})$, we have $\frac{X}{2} < \frac{1}{T_{max}} \leq  \frac{1}{T}, \forall T \in [1,2,\dots, T_{max}]$. $f_2$ monotonically decrease when $T < \frac{2}{X}$. 
            \begin{equation}
                \begin{aligned}
                    f_2(T|X) = \sqrt{\frac{kT^2}{K(XT-1)}} & \geq f_2(T=T_{max}|X) \\
                    & = \sqrt{\frac{kT_{max}^2}{K(XT_{max}-1)}}\\
                \end{aligned}
            \end{equation}
            Therefore, given $X \in [\frac{1}{T_{max}},\frac{2}{T_{max}})$, $(X, \sqrt{\frac{kT_{max}^2}{K(XT_{max}-1)}})$ Pareto dominates the set $\{(X, \sqrt{\frac{kT^2}{K(XT-1)}})| T \neq T_{max},\  \forall T \} \triangleq S_1$.
            
        \item Given $X \in [\frac{2}{T_{max}}, \frac{2k\sigma_{max}^2}{q}]$, $f_2$ reaches minimum $\frac{2\sqrt{k}}{\sqrt{K}X}$ when $\sigma^2 = \frac{qKX}{2k}$ according to inequality of arithmetic means.
            \begin{equation} 
                \begin{aligned}
                f_2(\sigma^2|X) & =\sqrt{\frac{q}{\sigma^2(X-k \frac{\sigma^2}{qK})}} \\
                & \geq f_2(\sigma^2=\frac{q K X}{2k}|X) = \frac{2\sqrt{k}}{\sqrt{K}X}. \\
                \end{aligned}
            \end{equation} 
            Therefore, given $X \in [\frac{2}{T_{max}}, \frac{2k\sigma_{max}^2}{q}]$,
            $(X, \frac{2\sqrt{k}}{\sqrt{K}X})$ Pareto dominates the set $\{(X, \sqrt{\frac{q}{\sigma^2(X-k\sigma^2/(qK))}})| \sigma^2 \neq \frac{qKX}{2k}, \forall \sigma \} \triangleq S_2$.
            
        \item Given $X \in (\frac{2k\sigma_{max}^2}{qK}, 1+\frac{k\sigma_{max}^2}{qK}]$, we have $\frac{qKX}{2k} > \sigma_{max}^2 \geq \sigma^2$.  $f_2$ monotonically decrease when $\sigma^2 < \frac{qKX}{2k}$
            \begin{equation}
                \begin{aligned}
                    f_2(\sigma^2|X) =\sqrt{\frac{q}{\sigma^2 (X-k\sigma^2/(qK))}} \geq f_2(\sigma=\sigma_{max}|X)& \\
                    = \sqrt{\frac{q}{\sigma_{max}^2(X-k\sigma_{max}^2/(qK))}}& \\
                \end{aligned}
            \end{equation}
            Therefore, given $X \in (\frac{2k\sigma_{max}^2}{qK}, 1+\frac{k\sigma_{max}^2}{qK}]$, $(X, \sqrt{\frac{q}{\sigma_{max}^2(X-k\sigma_{max}^2/(qK))}})$ Pareto dominates the set $\{(X, \sqrt{\frac{q}{\sigma^2(X-k\sigma^2/(qK))}})| \sigma^2 \neq \sigma_{max}, \forall \sigma \} \triangleq S_3$.
    \end{itemize}
    
    As a result, the Pareto set of Eq. \eqref{equation: simplifiedMOO} must be a subset of $S_1 \cup S_2 \cup S_3 \triangleq S$.
    Since each points in $S$ non-dominates each other, the set $S$ is the Pareto set of Eq. \eqref{equation: simplifiedMOO}. We derive the Pareto solution set based on the Pareto front in the following.
    \begin{itemize}
        \item For $X \in [\frac{1}{T_{max}},\frac{2}{T_{max}})$, it reaches Pareto front when $T=T_{max}$. We have the following Pareto solutions.
            \begin{equation}
                \forall \sigma \in [0,\sqrt{\frac{qK}{k T_{max}}}], T=T_{max}
            \end{equation}
        \item For $X \in [\frac{2}{T_{max}}, \frac{2k\sigma_{max}^2}{qK}]$, it reaches Pareto front when $\sigma=\sqrt{\frac{qK X}{2k}}$. As $\sigma=\sqrt{\frac{q K X}{2k}}$ is equivalent to $\sigma^2 T= \frac{qK}{k}$, we have the following Pareto solutions.
            \begin{equation}
                \sigma = \sqrt{\frac{qK}{kT}}\ \ when\ T \in \{\lceil \frac{qK}{k\sigma_{max}^2} \rceil, \dots, T_{max}-1\}
            \end{equation}
        \item For $X \in (\frac{2k\sigma_{max}^2}{qK}, 1+\frac{k\sigma_{max}^2}{qK}]$, it reaches Pareto front when $\sigma=\sigma_{max}$. We have the following Pareto solutions.
            \begin{equation}
                \sigma = \sigma_{max}\ \ when\  T \in \{1,\dots,\lfloor \frac{qK}{k\sigma_{max}^2} \rfloor\}
            \end{equation}
    \end{itemize}

    \XY{As $T_{max} = \lfloor\frac{\Bar{\epsilon_r}}{t_c}\rfloor$, we have done the proof of Case-II.}

    \textbf{Case-III (Constrained $\mathbf{\sigma}$ with $\mathbf{k \sigma_{max}^2 \lfloor\frac{\Bar{\epsilon_r}}{t_c}\rfloor \leq qK}$)} \XY{As $T_{max}=\lfloor\frac{\Bar{\epsilon_r}}{t_c}\rfloor$, we have $T \in [1,\dots,T_{max}]$ and $k \sigma_{max}^2 T_{max} \leq {qK}$ with $\sigma \in [0,\sigma_{max}]$ and given $q$} We separate the value of $X$ into two parts.

    \begin{itemize}
        \item Given $X \in [\frac{1}{T_{max}},\frac{1}{T_{max}}+\frac{k \sigma_{max}^2}{qK})$, we have 
            \begin{equation}
                X < \frac{1}{T_{max}}+\frac{k \sigma_{max}^2}{qK} < \frac{1}{T_{max}}+\frac{k}{qK}\frac{qK}{kT_{max}},
            \end{equation}
            which means $T \leq T_{max} < \frac{2}{X}$. $f_2$ monotonically decrease when $T < \frac{2}{X}$. 
            \begin{equation}
                \begin{aligned}
                    f_2(T|X) = \sqrt{\frac{kT^2}{K(XT-1)}} & \geq f_2(T=T_{max}|X) \\
                    & = \sqrt{\frac{kT_{max}^2}{K(XT_{max}-1)}}\\
                \end{aligned}
            \end{equation}
            Therefore, given $X \in [\frac{1}{T_{max}},\frac{2}{T_{max}})$, $(X, \sqrt{\frac{kT_{max}^2}{K(XT_{max}-1)}})$ Pareto dominates the set $\{(X, \sqrt{\frac{kT^2}{K(XT-1)}})| T \neq T_{max}, \forall T \} \triangleq S_1$.

        \item Given $X \in [\frac{1}{T_{max}}+\frac{k \sigma_{max}^2}{qK}, 1+\frac{k \sigma_{max}^2}{qK}]$, we have 
            \begin{equation}
                X \geq \frac{1}{T_{max}}+\frac{k \sigma_{max}^2}{q} > \frac{k \sigma_{max}^2}{q} + \frac{k \sigma_{max}^2}{q},
            \end{equation}
            which means $\sigma^2 < \frac{qKX}{2k}$.
            $f_2$ monotonically decrease when $\sigma^2 < \frac{qKX}{2k}$
            \begin{equation}
                \begin{aligned}
                    f_2(\sigma^2|X) =\sqrt{\frac{q}{\sigma^2 (X-k\sigma^2/(qK))}} \geq f_2(\sigma=\sigma_{max}|X)& \\
                    = \sqrt{\frac{q}{\sigma_{max}^2(X-k\sigma_{max}^2/(qK))}}& \\
                \end{aligned}
            \end{equation}
            Therefore, given $X \in [\frac{1}{T_{max}}+\frac{k \sigma_{max}^2}{qK}, 1+\frac{k \sigma_{max}^2}{qK}]$, $(X, \sqrt{\frac{q}{\sigma_{max}^2(X-k\sigma_{max}^2/(qK))}})$ Pareto dominates the set $\{(X, \sqrt{\frac{q}{\sigma^2(X-k\sigma^2/(qK))}})| \sigma^2 \neq \sigma_{max}, \forall \sigma \} \triangleq S_2$.
    \end{itemize}
    
    As a result, the Pareto set of Eq. \eqref{equation: simplifiedMOO} must be a subset of $S_1 \cup S_2 \triangleq S$.
    Since each points in $S$ non-dominates each other, the set $S$ is the Pareto set of Eq. \eqref{equation: simplifiedMOO}. We derive the Pareto solution set based on the Pareto front in the following.
    % We denote S as $$S = \{(f_1, f_2^{\prime})\}, \text{ for all } X,$$ where $f_2^{\prime}$ is as following.
    % \begin{equation}
    %     \left\{
    %     \begin{aligned}
    %         & \sqrt{\frac{kT_{max}^2}{XT_{max}-1}}, X \in [\frac{1}{T_{max}},\frac{1}{T_{max}}+
    %         \frac{k \sigma_{max}^2}{q})\\
    %         & \sqrt{\frac{q}{\sigma_{max}^2(X-k\sigma_{max}^2/q)}},\\
    %         & \ \ \ \ \ \ \  \ \ \ \ \ \ \ X \in [\frac{1}{T_{max}}+\frac{k \sigma_{max}^2}{q}, 1+\frac{k\sigma_{max}^2}{q}] \\
    %     \end{aligned}
    %     \right.
    % \end{equation}
    
    % As $(f_1, f_2^{\prime})$ Pareto dominates $(f_1, f_2)$ where $f_2 \neq f_2^{\prime}$ for each $X$, the Pareto set must be a subset of $S$.
    % Having the points in $S$ non-dominate each other, the set $S$ is the Pareto set.
    % We can derive the Pareto solution set based on the Pareto front.
    \begin{itemize}
        \item For $X \in [\frac{1}{T_{max}},\frac{1}{T_{max}}+\frac{k \sigma_{max}^2}{qK})$, it reaches Pareto front when $T=T_{max}$. We have the following Pareto solutions.
            \begin{equation}
                \forall \sigma \in [0,\sigma_{max}],\ T = T_{max}
            \end{equation}
        \item For $X \in [\frac{1}{T_{max}}+\frac{k \sigma_{max}^2}{qK}, 1+\frac{k \sigma_{max}^2}{qK}]$, it reaches Pareto front when $\sigma = \sigma_{max}$. We have the following Pareto solutions.
            \begin{equation}
                \sigma = \sigma_{max},\ when\ T \in \{1,2,\dots,T_{max}-1\}
            \end{equation}
    \end{itemize}
    
    \XY{As $T_{max} = \lfloor\frac{\Bar{\epsilon_r}}{t_c}\rfloor$, we have done the proof of Case-III.}

    \XY{We have done the proof of the relationship between Pareto optimal $\sigma$, $T$, and $q$ of all the three cases. The theoretical analysis provides a theoretical guarantee for guiding optimal parameter design with low cost in DPFL.}
\end{proof}

\section{Experiment} \label{Section: experiment}

\begin{figure*}[ht]
\vspace{-10pt}
    \centering
    \subfigure[Pareto Solution for LR]{
        \includegraphics[scale=0.45]{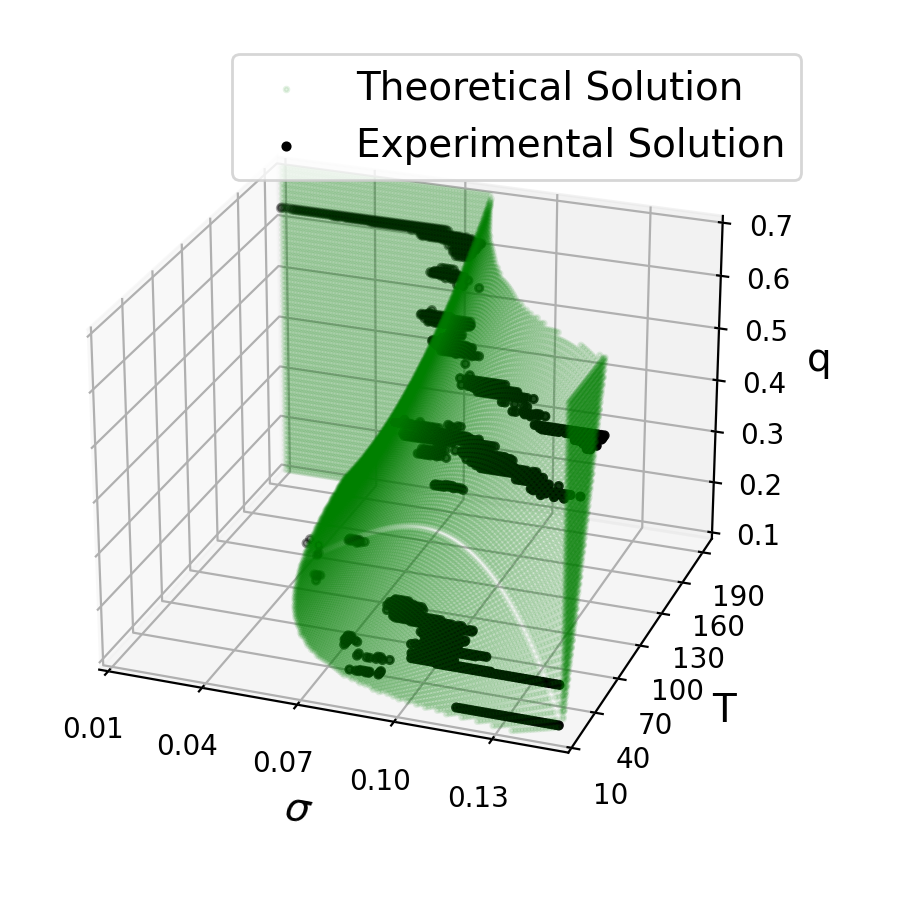} 
        \label{Fig: compare LR}
    }
    \subfigure[Pareto Solution for LeNet]{
        \includegraphics[scale=0.45]{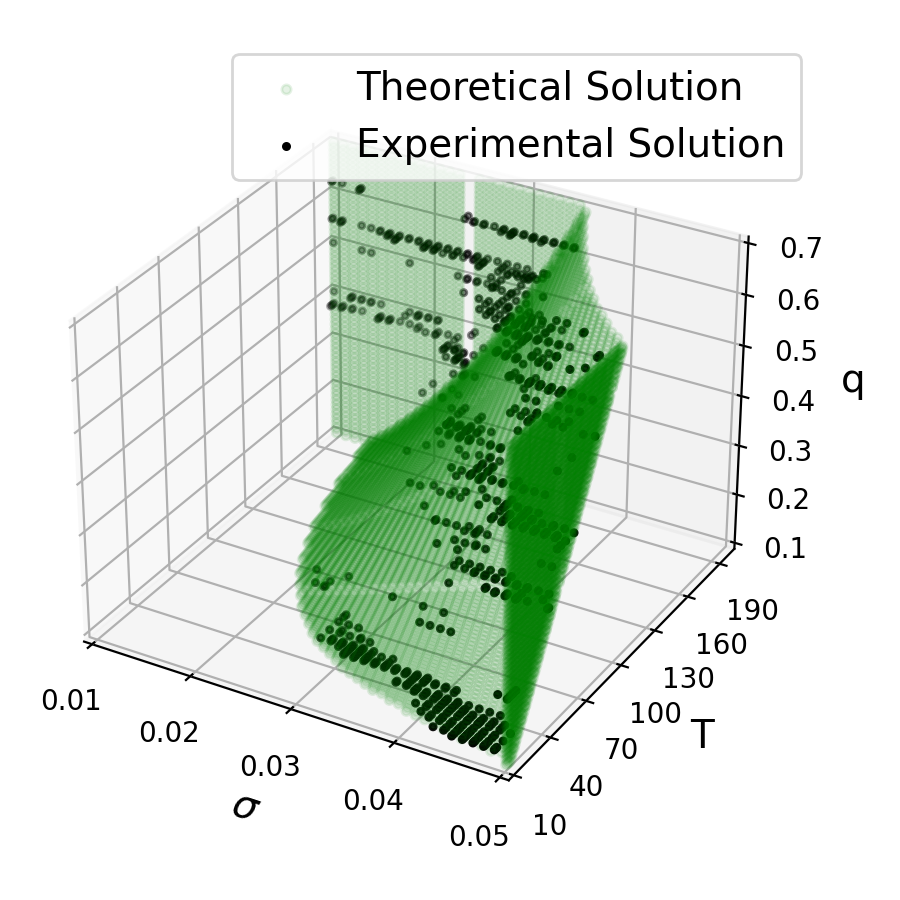} 
        \label{Fig: compare LeNet}
    }
    \subfigure[Pareto Solution for ResNet-18]{
        \includegraphics[scale=0.45]{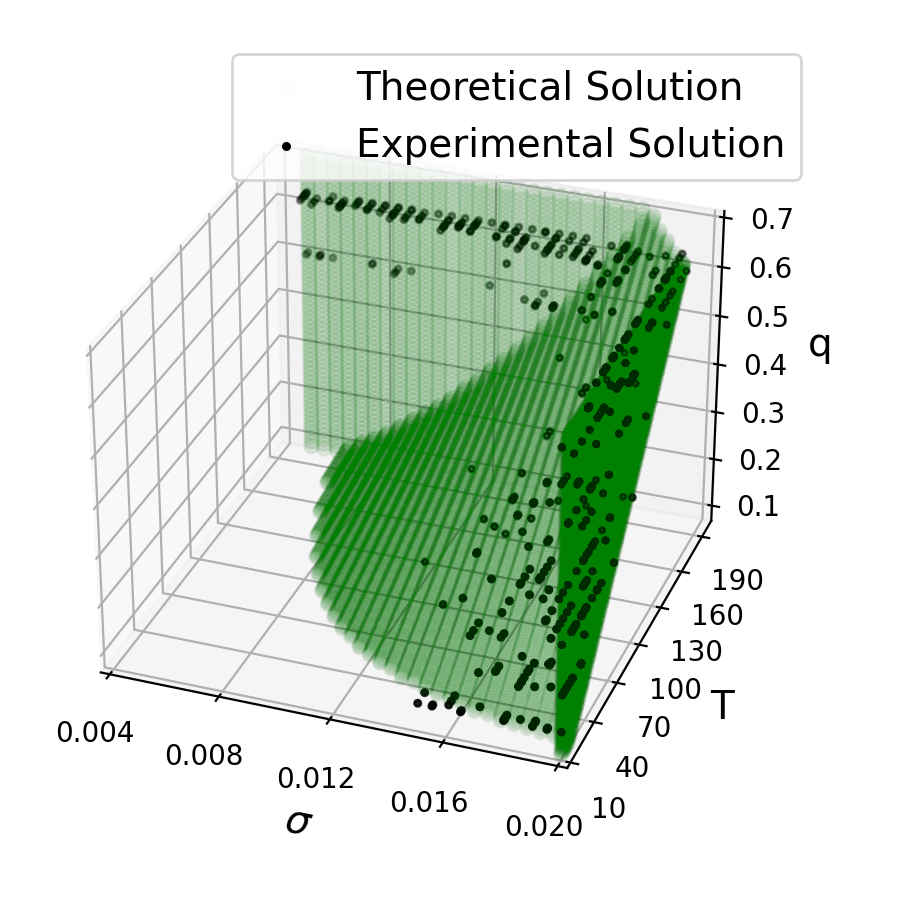} 
        \label{Fig: compare ResNet}
    }
    
    \caption{\textbf{Pareto Front and Pareto Solution for LeNet and LR with respect to $\sigma$, $T$, and $q$.} 
    With total $40$ clients and $20$ local epochs in LR model, Fig. \ref{Fig: compare LR} compares the theoretical solutions with the experimental solutions. The green dots show the theoretical solutions with constant $k=22$ which means $22 \sigma^2 T = 40 q$. The black dots show the experimental solutions by Algo. \ref{algorithm: non-dominated sorting} with $q \in [0.125, 0.2, 0.375, 0.4, 0.45, 0.55, 0.625]$.
    With total $40$ clients and $20$ local epochs in LeNet model, Fig. \ref{Fig: compare LeNet} compares the theoretical solutions with the experimental solutions. The green dots show the theoretical solutions with constant $k=110$ which means $110 \sigma^2 T = 40 q$. The black dots show the experimental solutions by Algo. \ref{algorithm: non-dominated sorting} with $q \in [0.125, 0.2, 0.25, 0.3, 0.375, 0.4, 0.5, 0.55, 0.625]$.
    With total $40$ clients and $20$ local epochs in ResNet-18 model, Fig. \ref{Fig: compare ResNet} compares the theoretical solutions with the experimental solutions. The green dots show the theoretical solutions with constant $k=440$ which means $440 \sigma^2 T = 40 q$. The black dots show the experimental solutions by Algo. \ref{algorithm: non-dominated sorting} with $q \in [0.125, 0.2, 0.25, 0.3, 0.375, 0.5, 0.625]$.}
    \label{Fig:Pareto Front and Pareto Solution for LeNet and LR}
    \vspace{-10pt}
\end{figure*}

In this section, we use experiments to verify our theoretical analysis. 
Firstly, we display the Pareto solutions of the \XY{efficiency} constrained \XY{utility-privacy} bi-objective optimization problem in DPFL regarding noise level ($\sigma$), communication rounds ($T$), and sample ratio ($q$) in Sec.\ref{sec: 3factorexperiment}.
Secondly, we further investigate the influence of the total number of participating clients ($K$) and the local training epochs ($E$) on Pareto \XY{solutions} in Sec.\ref{sec: ablation study}.
Finally, we demonstrate the process of \XY{low cost parameter design guiding} by Thm. \ref{Thm: analytical solution} and Cor. \ref{Cor: analytical solution with given q} in Sec.\ref{sec:parameter design}.

\subsection{Experimental Setup}
We implement both LR (logistic regression) and LeNet \cite{lecun1998lenet} on the MNIST \cite{deng2012mnist} dataset, \ADD{and employ ResNet-18 \cite{he2016deep} on the CIFAR10 \cite{krizhevsky2009learning} dataset} to verify our theoretical analysis. 
The MNIST dataset includes $60000$ training samples and $10000$ testing samples. \ADD{The CIFAR-10 dataset includes $50000$ training samples and $10000$ testing samples.} The samples are identically divided to $K$ parts and kept locally within each client.

The communication rounds $T_{max}$ is set to be $200$.
For LR model, the range of $\sigma$ is set to be within $[0.010, 0.150]$. For LeNet model, the range of $\sigma$ is set to be within $[0.010, 0.050]$. \ADD{For ResNet-18 model, the range of $\sigma$ is set to be within $[0.0005, 0.0200]$.}
Having batch size $B=64$, we use stochastic gradient descent optimizer with learning rate $\eta = 0.01$ and \ADD{momentum as $0.09$}.
To better estimate the test loss, we use multiple random seeds (30 for LR, 50 for LeNet, \ADD{36 for ResNet-18}) and take average among the test loss of different seeds for different given $\sigma$, $T$, and $q$. 

\subsection{Pareto Solution in DPFL}\label{sec: 3factorexperiment}

\begin{figure*}[htbp]
    \centering
    \subfigure[\XY{Case-I ($\sigma \in [0,+\infty)$)} for LR]{
        \includegraphics[scale=0.47]{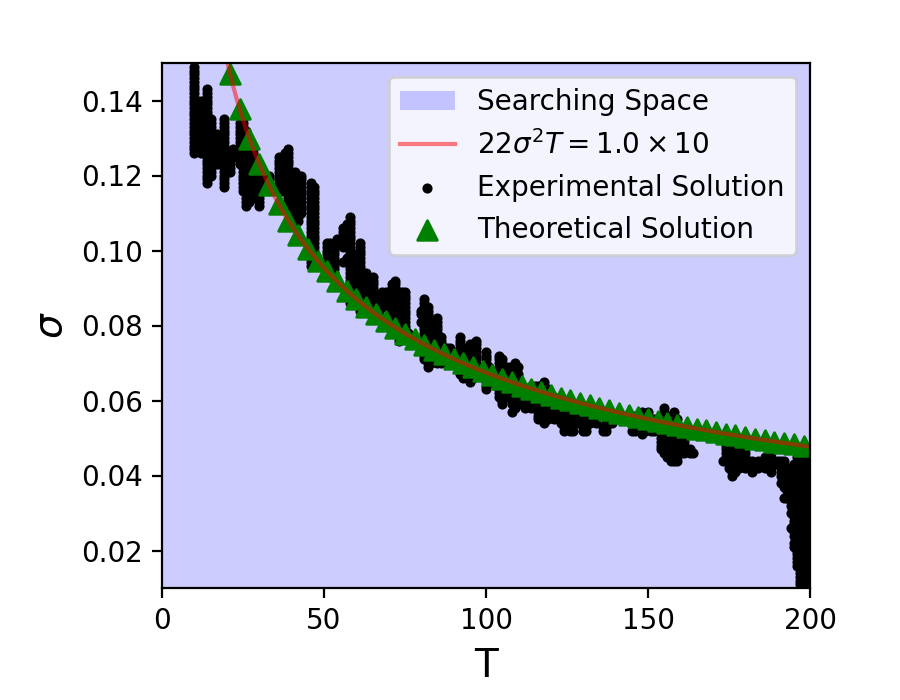} 
        \label{Fig: LR no constraint}
    }
    \subfigure[\XY{Case-II (Constrained $\sigma$ with $k \sigma_{max}^2 \lfloor\frac{\Bar{\epsilon_r}}{t_c}\rfloor > {qK}$)} for LR]{
        \includegraphics[scale=0.47]{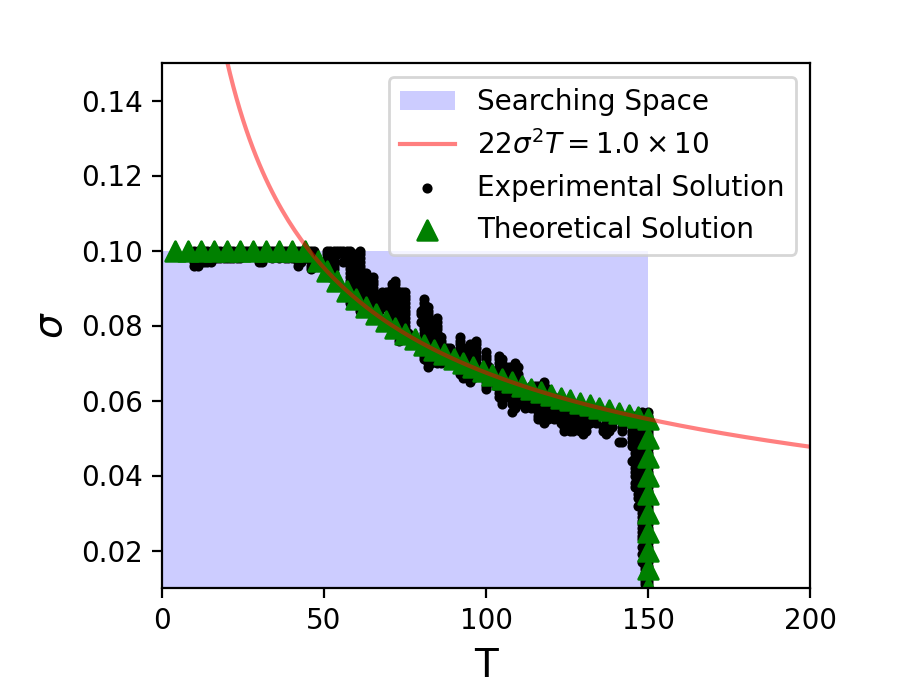} 
        \label{Fig: LR wide range}
    }
    \subfigure[\XY{Case-III (Constrained $\sigma$ with $k \sigma_{max}^2 \lfloor\frac{\Bar{\epsilon_r}}{t_c}\rfloor \leq qK$)} for LR]{
        \includegraphics[scale=0.47]{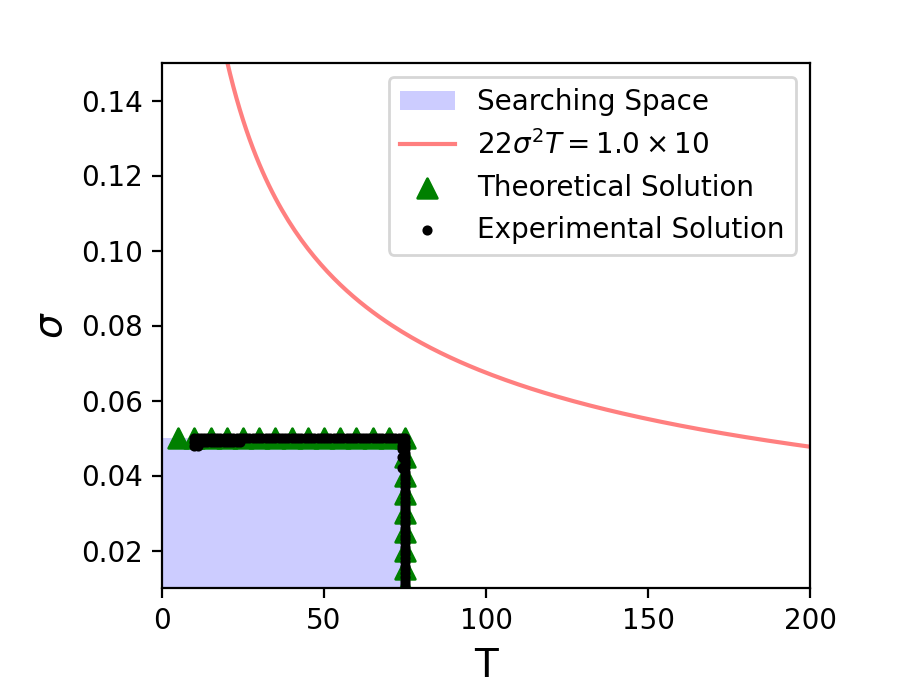}
        \label{Fig: LR small range}
    }

    \subfigure[\XY{Case-I ($\sigma \in [0,+\infty)$)} for LeNet]{
        \includegraphics[scale=0.47]{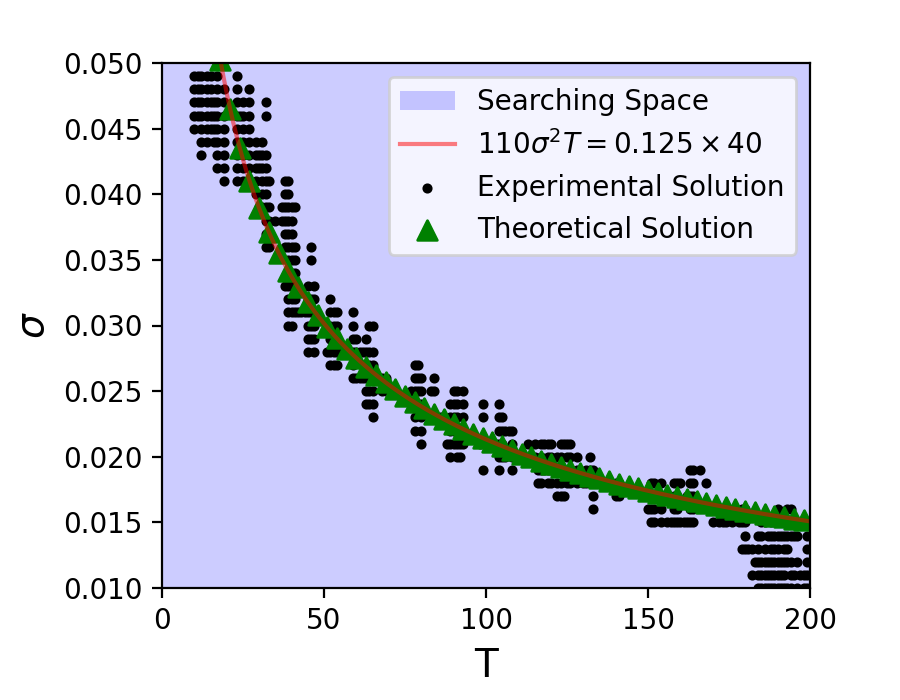} 
        \label{Fig: LeNet no constraint}
    }
    \subfigure[\XY{Case-II (Constrained $\sigma$ with $k \sigma_{max}^2 \lfloor\frac{\Bar{\epsilon_r}}{t_c}\rfloor > {qK}$)} for LeNet]{
        \includegraphics[scale=0.47]{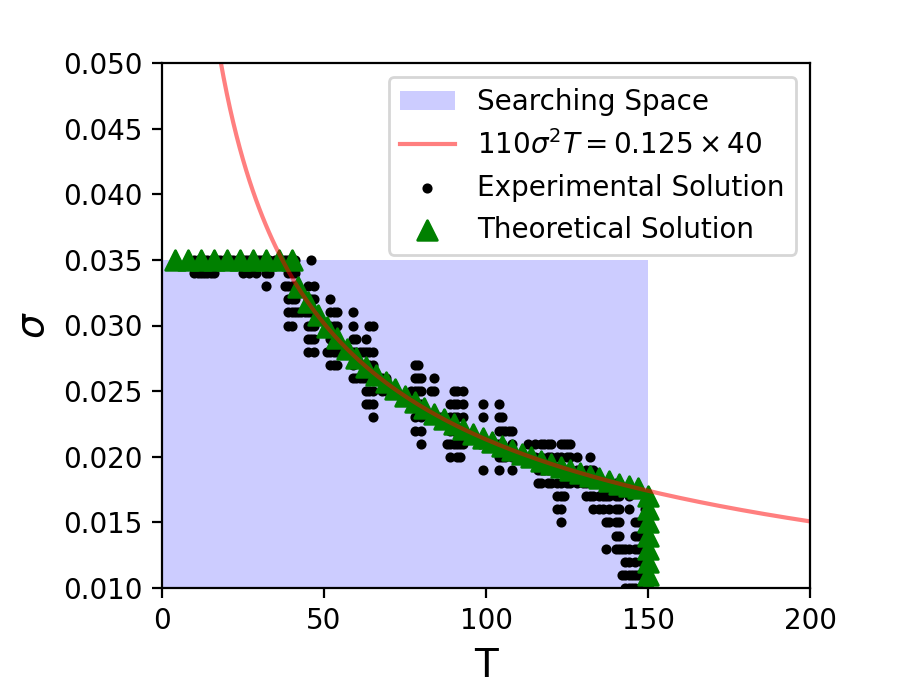}
        \label{Fig: LeNet wide range}
    }
    \subfigure[\XY{Case-III (Constrained $\sigma$ with $k \sigma_{max}^2 \lfloor\frac{\Bar{\epsilon_r}}{t_c}\rfloor \leq qK$)} for LeNet]{
        \includegraphics[scale=0.47]{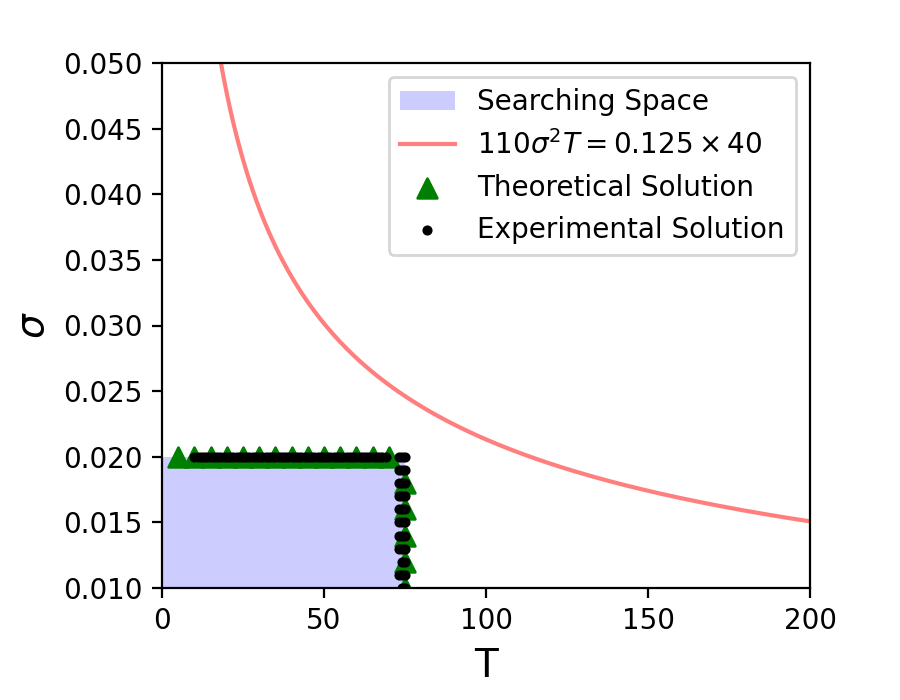}
        \label{Fig: LeNet small range}
    }

    \caption{\textbf{Pareto solution for LR and LeNet model with different searching space.} 
    With total $K=10$, $q=1.0$ and $E=20$ in LR model, Fig. \ref{Fig: LR no constraint} shows the Pareto solution in \XY{Case-I where we assume $\sigma \in [0.01,0.15]$ and $T \in [1,\dots,200]$ with $\lfloor\frac{\Bar{\epsilon_r}}{t_c}\rfloor = 200$}. Fig. \ref{Fig: LR wide range} shows the Pareto solution in the Case-II as $k\sigma_{max}^2 \lfloor\frac{\Bar{\epsilon_r}}{t_c}\rfloor > qK$ where $\lfloor\frac{\Bar{\epsilon_r}}{t_c}\rfloor = 150$ and $\sigma_{max} = 0.10$. Fig. \ref{Fig: LR small range} shows the Pareto solution in Case-III as $k\sigma_{max}^2 \lfloor\frac{\Bar{\epsilon_r}}{t_c}\rfloor \leq qK$ where $\lfloor\frac{\Bar{\epsilon_r}}{t_c}\rfloor = 75$ and $\sigma_{max} = 0.50$.
    With total $K=40$, $q=0.125$ and $E=20$ in LR model, Fig. \ref{Fig: LeNet no constraint} shows the Pareto solution in \XY{Case-I where we assume $\sigma \in [0.01,0.05]$ and $T \in [1,\dots,200]$ with $\lfloor\frac{\Bar{\epsilon_r}}{t_c}\rfloor = 200$}. Fig. \ref{Fig: LeNet wide range} shows the Pareto solution in the wide range case as $k\sigma_{max}^2 \lfloor\frac{\Bar{\epsilon_r}}{t_c}\rfloor > qK$ where $\lfloor\frac{\Bar{\epsilon_r}}{t_c}\rfloor = 150$ and $\sigma_{max} = 0.035$. Fig. \ref{Fig: LeNet small range} shows the Pareto solution in the small range case as $k\sigma_{max}^2 \lfloor\frac{\Bar{\epsilon_r}}{t_c}\rfloor \leq qK$ where $\lfloor\frac{\Bar{\epsilon_r}}{t_c}\rfloor = 75$ and $\sigma_{max} = 0.02$.}
    \label{Fig:LR_LeNet_different range}
\end{figure*}

% We display the Pareto solution and front from theoretical and experimental perspectives in Fig. \ref{Fig:Pareto Front and Pareto Solution for LeNet and LR}. Specifically, we search for the Pareto solution via non-dominated sorting as Algo. \ref{algorithm: non-dominated sorting} in order to obtain the experimental Pareto solution (black point) and compare it with the theoretical solution of Thm. \ref{Thm: analytical solution} (green line) illustrated in Fig. \ref{Fig: compare LR} and Fig. \ref{Fig: compare LeNet}. 
% It shows the experimental Pareto solution matches the theoretical Pareto solution.
% To clear represent the analytical solution, we draw the surface of $\sigma^2 T = \frac{q}{K}$ in Fig. \ref{Fig: Pareto solution for LR} and Fig. \ref{Fig: Pareto solution for LeNet}. Fig. \ref{Fig: Pareto front for LR} and \ref{Fig: Pareto front for LeNet} shows the corresponding Pareto front.
We present the Pareto solutions from both theoretical and experimental perspectives, as depicted in Fig. \ref{Fig:Pareto Front and Pareto Solution for LeNet and LR}. In our approach, we employ non-dominated sorting, as outlined in Algo. \ref{algorithm: non-dominated sorting}, to identify the experimental Pareto solution represented by the black points. We then compare this experimental solution on MNIST and CIFAR-10 with the theoretical solution derived in Thm. \ref{Thm: analytical solution}, illustrated separately in Fig. \ref{Fig: compare LR}, Fig. \ref{Fig: compare LeNet}, and \ADD{\ref{Fig: compare ResNet}}.
Notably, Fig. \ref{Fig:Pareto Front and Pareto Solution for LeNet and LR} reveals a strong alignment between the experimental and theoretical Pareto solutions, signifying the accuracy of our theoretical model.

Furthermore, to provide additional confirmation of the Pareto solutions derived in Cor. \ref{Cor: analytical solution with given q}, we present a comparison of the theoretical and experimental Pareto solutions for different cases with a fixed $q$, as illustrated in Fig. \ref{Fig:LR_LeNet_different range}.

Across all \XY{the three cases mentioned in Thm. \ref{Thm: analytical solution}}, the Pareto solutions obtained through Algo. \ref{algorithm: non-dominated sorting} (depicted as black points) align closely with the theoretical solutions (represented by green \XY{surface}) established in Cor. \ref{Cor: analytical solution with given q}. This consistency serves as strong evidence of the validity and reliability of our analytical model in predicting the Pareto front under various conditions.

\subsection{Ablation Study  on $K, E$} \label{sec: ablation study}
In this subsection, we investigate the impact of local epochs ($E$) and the number of clients ($K$) on the Pareto solutions with a fixed sample ratio ($q$), as discussed in Cor. \ref{Cor: analytical solution with given q}. Our analysis yields the following three key conclusions:
\begin{itemize}
    \item \textbf{Local Epochs ($E$) Exhibit Minimal Influence:} We find that the local epochs ($E$) have a negligible impact on the Pareto solutions. This observation is consistent across cases where $E$ is set to different values, such as $E=5, 10, 20$, as demonstrated in Fig. \ref{Fig:Local Epochs Makes No Difference}. This similarity in Pareto solutions indicates that the choice of $E$ does not significantly alter \XY{the parameters for achieving the optimal} trade-off between privacy and utility.

    \item \textbf{Number of Clients ($K$) Affects Pareto Solution directly:} We observe that changes in the experimental Pareto solution concerning $\sigma$ and $T$ (as indicated by the black points in Figure \ref{Fig:Number of Participating Clients makes difference}) exhibit an direct relationship with the number of participating clients, denoted as $K$. Specifically, the relationship between $\sigma$ and $T$ follows the equation $k \sigma^2 T = {qK}$, which is represented by the green line in Figure \ref{Fig:Number of Participating Clients makes difference}. This relationship implies that as the number of clients increases, the values of $\sigma$ and $T$ must adjust accordingly to maintain consistent Pareto solutions.
  
\end{itemize}

These findings provide valuable insights into the factors that influence the Pareto solutions in the context of federated learning with differential privacy, enabling practitioners to make informed decisions when designing privacy-preserving algorithms.

% In Fig. \ref{Fig:Local Epochs Makes No Difference}, we display the theoretical and experimental Pareto solution for both LR and LeNet model with different local epochs ($E$). With given sample ratio $q=1.0$ and total $K=10$ clients, the theoretical Pareto solution from Cor. \ref{Cor: analytical solution with given q} is consistent with the experimental solution via Algo. \ref{algorithm: non-dominated sorting}. The local epochs ($E$) has no significant influence on the Pareto solution as the Pareto solutions are similar for cases with $E=5,10,20$.

\begin{figure*}
    \centering
    % \vspace{-0.5cm}
    \subfigure[Local Epochs $E=5$ (LR)]{
        \includegraphics[scale=0.47]{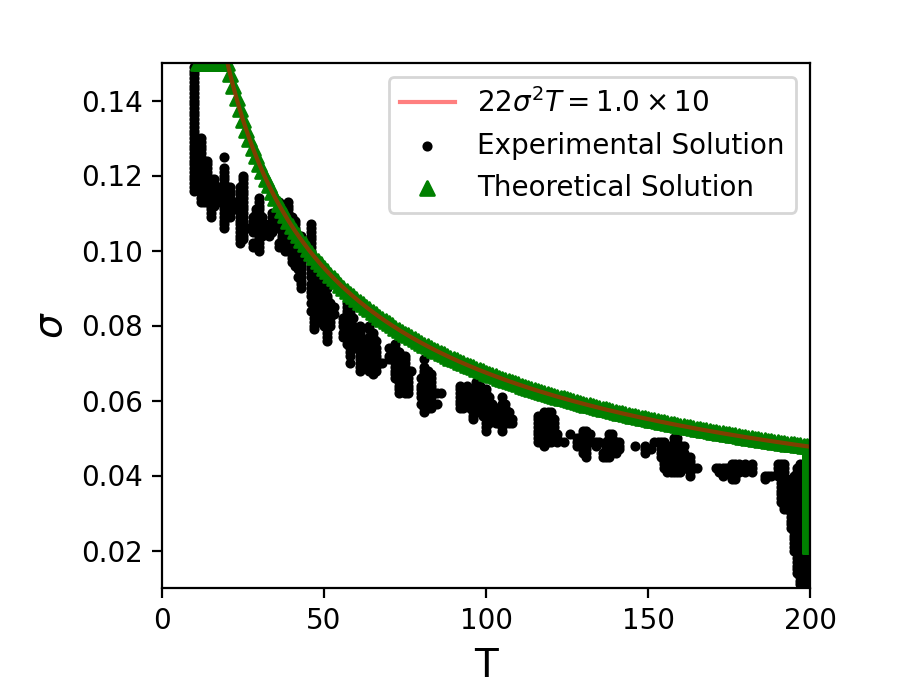} 
        \label{figLR:E=5}
    }
    \subfigure[Local Epochs $E=10$ (LR)]{
        \includegraphics[scale=0.47]{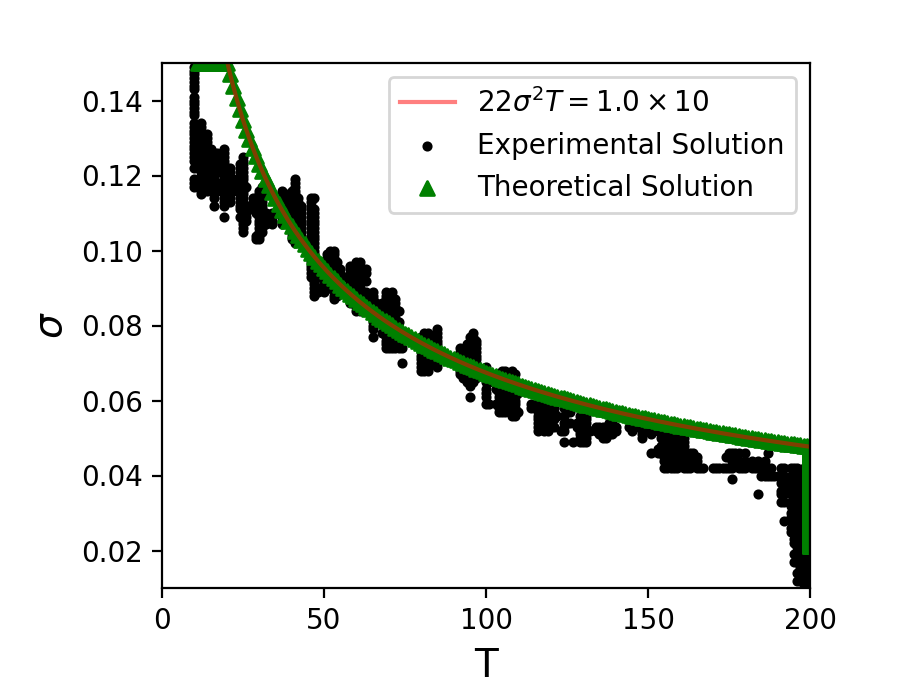} 
        \label{figLR:E=10}
    }
    \subfigure[Local Epochs $E=20$ (LR)]{
        \includegraphics[scale=0.47]{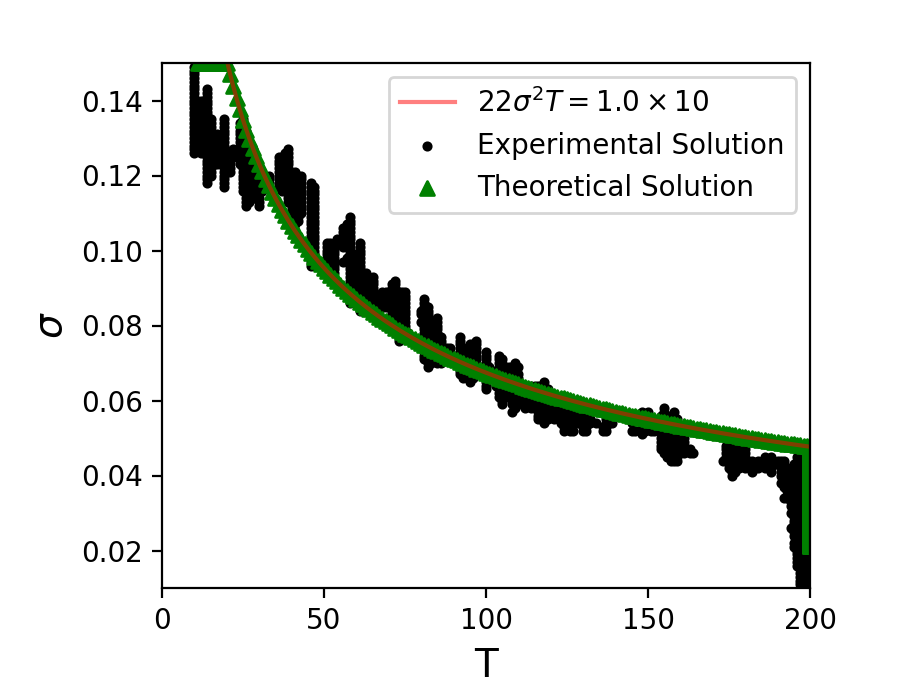}
        \label{figLR:E=20}
    }

    \subfigure[Local Epochs $E=5$ (LeNet)]{
        \includegraphics[scale=0.47]{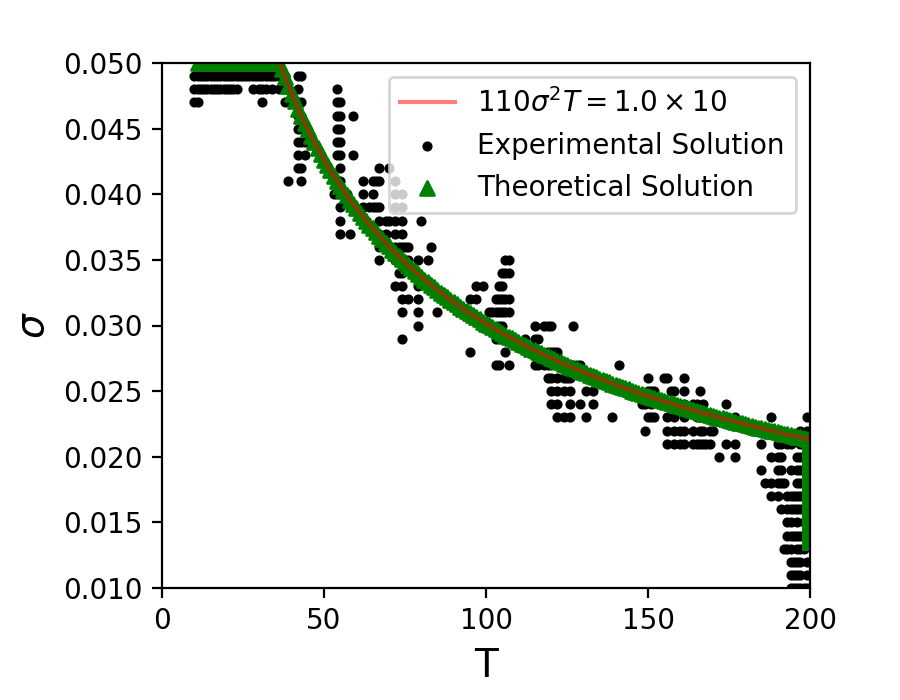} 
        \label{figLeNet:E=5}
    }
    \subfigure[Local Epochs $E=10$ (LeNet)]{
        \includegraphics[scale=0.47]{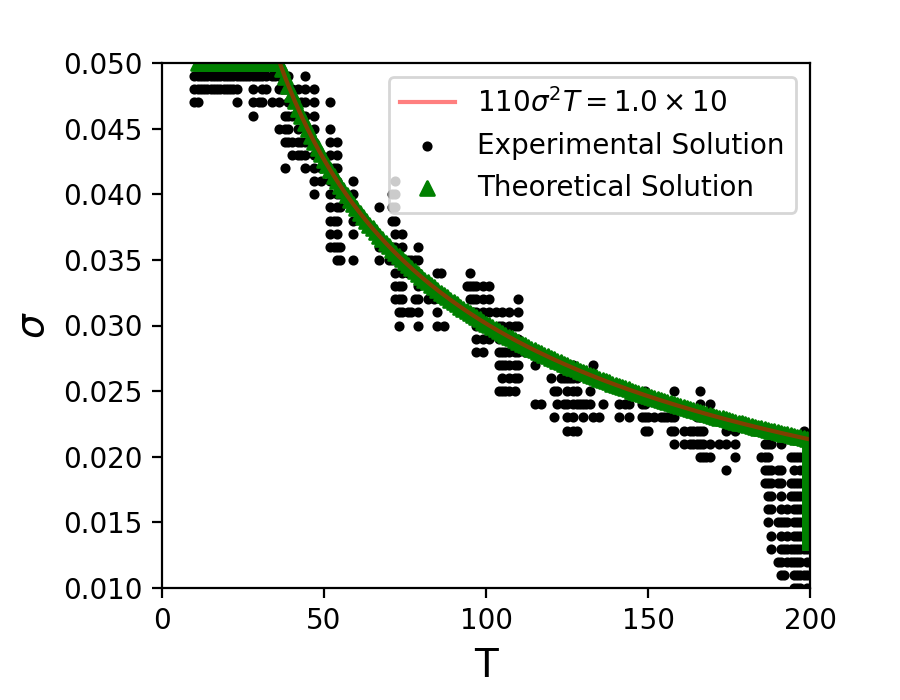} 
        \label{figLeNet:E=10}
    }
    \subfigure[Local Epochs $E=20$ (LeNet)]{
        \includegraphics[scale=0.47]{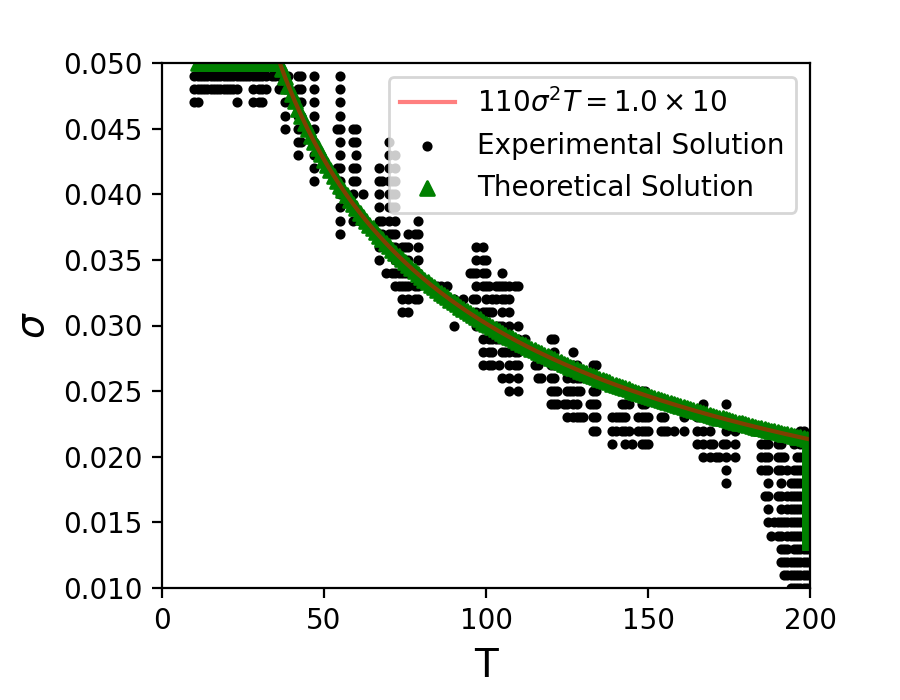}
        \label{figLeNet:E=20}
    }
    \DeclareGraphicsExtensions.
    \caption{\textbf{Pareto solutions for LR and LeNet with different local epochs.} 
    With $K=10$ and $q=1.0$, Fig \ref{figLR:E=5}, \ref{figLR:E=10}, and \ref{figLR:E=20} show the Pareto solution for LR with local epochs $E=5,10,20$ respectively.
    With $K=10$ and $q=1.0$, Fig \ref{figLeNet:E=5}, \ref{figLeNet:E=10}, and \ref{figLeNet:E=20} show the Pareto solution for LeNet with local epochs $E=5,10,20$ respectively.}
    \label{Fig:Local Epochs Makes No Difference}
\end{figure*}

% In Fig. \ref{Fig:Number of Participating Clients makes difference}, we display the theoretical and experimental Pareto solution for both LR and LeNet model with different total number of clients ($K$). With given sample ratio $q=1.0$, the theoretical Pareto solution from Cor. \ref{Cor: analytical solution with given q} is consistent with the experimental solution via Algo. \ref{algorithm: non-dominated sorting}. The value $k$ of equation $\sigma^2 T = {q}/{k}$ for $K=5, 10, 20$ are $5.0, 2.5, 1.25$ respectively in LR model and $22.0, 11.0, 5.5$ respectively in LeNet model, which follows the theoretical analysis as $q = {const.}/{K}$. 

\begin{figure*}
    \centering
    \subfigure[Total Clients $K=5$ (LR)]{
        \includegraphics[scale=0.47]{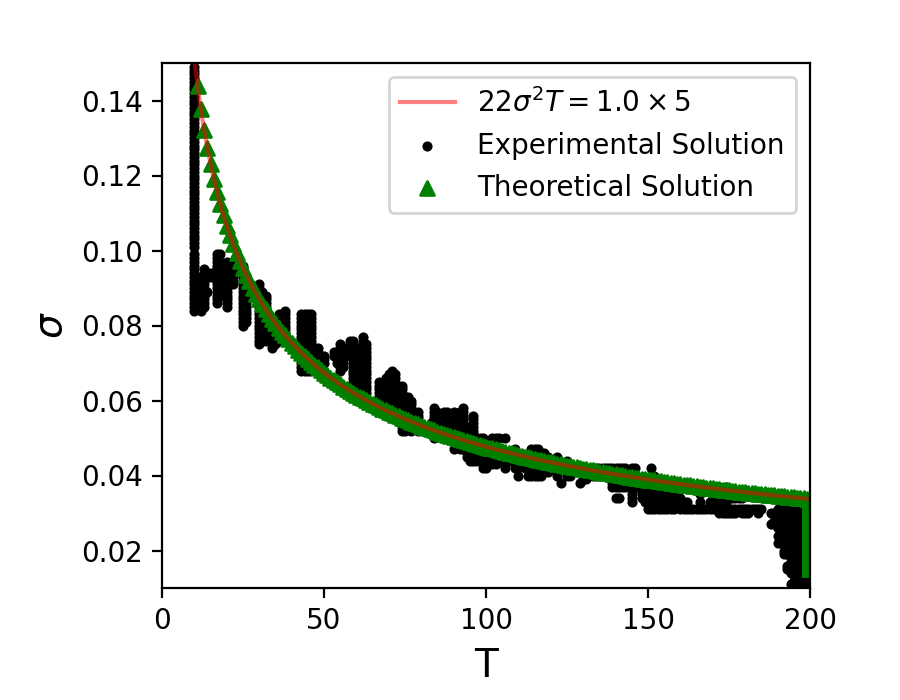} 
        \label{figLR:K=5}
    }
    \subfigure[Total Clients $K=10$ (LR)]{
        \includegraphics[scale=0.47]{Pictures/LR_10_1.0_200_20.png} 
        \label{figLR:K=10}
    }
    \subfigure[Total Clients $K=20$ (LR)]{
        \includegraphics[scale=0.47]{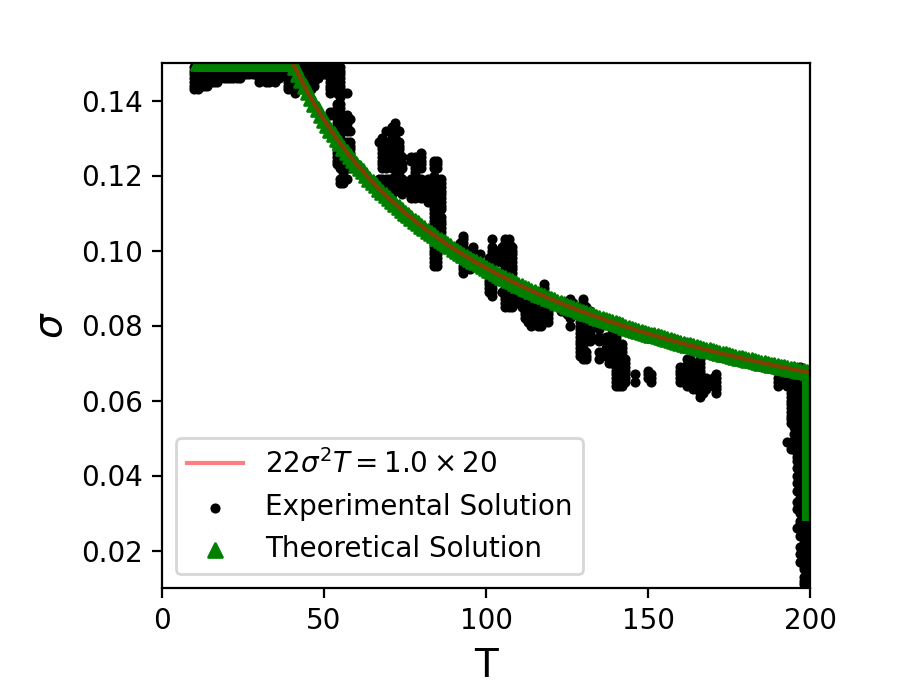}
        \label{figLR:K=20}
    }

    \subfigure[Total Clients $K=5$ (LeNet)]{
        \includegraphics[scale=0.47]{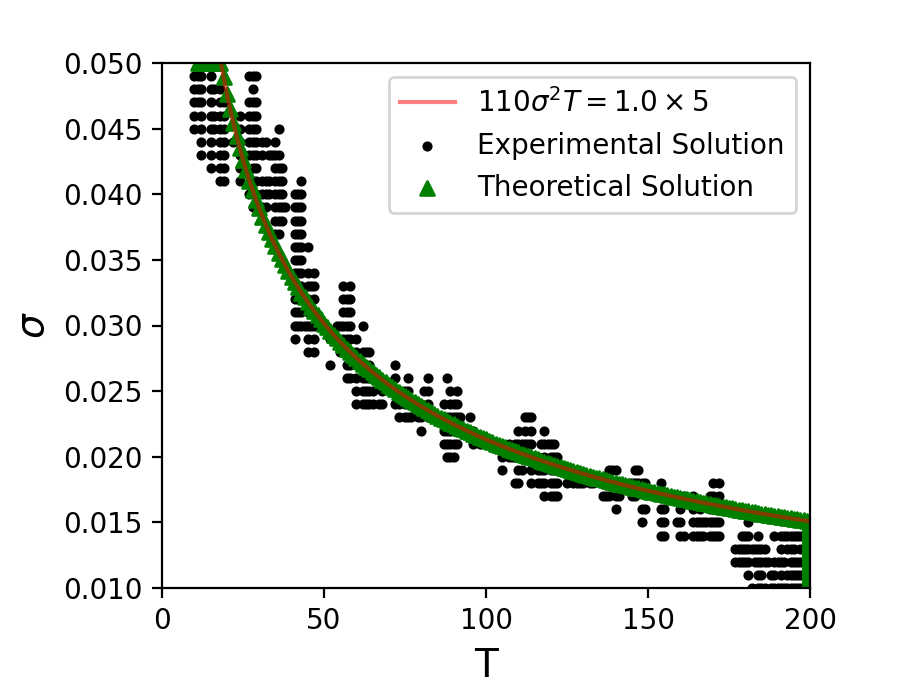} 
        \label{figLeNet:K=5}
    }
    \subfigure[Total Clients $K=10$ (LeNet)]{
        \includegraphics[scale=0.47]{Pictures/LeNet_10_1.0_200_20.png} 
        \label{figLeNet:K=10}
    }
    \subfigure[Total Clients $K=20$ (LeNet)]{
        \includegraphics[scale=0.47]{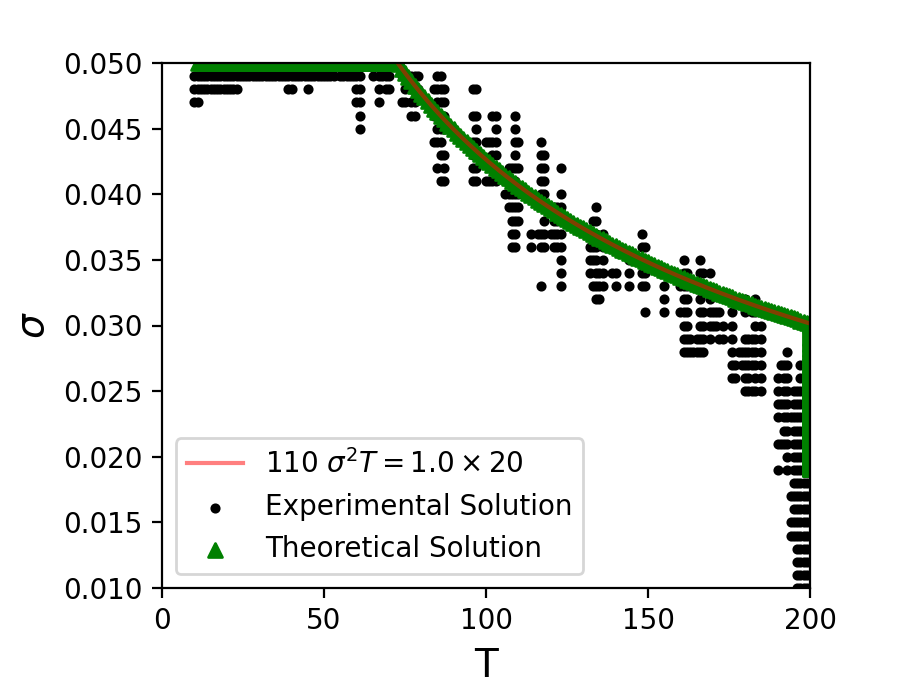}
        \label{figLeNet:K=20}
    }
    \DeclareGraphicsExtensions.
    \caption{\textbf{Pareto solutions for LR and LeNet with different number of total clients.} 
    With $q=1.0$ and $E=20$, Fig. \ref{figLR:K=5}, \ref{figLR:K=10}, \ref{figLR:K=20} show the Pareto solution for LR with number of total clients $K=5, 10, 20$ respectively. 
    With $q=1.0$ and $E=20$, Fig. \ref{figLeNet:K=5}, \ref{figLeNet:K=10}, \ref{figLeNet:K=20} show the Pareto solution for LR with number of total clients $K=5, 10, 20$ respectively.}
    \label{Fig:Number of Participating Clients makes difference}
\end{figure*}

\subsection{Demonstration of \XY{low cost} Parameter Design}\label{sec:parameter design}
In the real scenario, due to \XY{the privacy preserved limitation and training efficiency constraint}, it is not realistic to do a large number of experiments to search for the best parameters as communication rounds ($T$), noise level ($\sigma$) and sample ratio ($q$). In other words, making sure the privacy leakage and utility loss reach the Pareto front \XY{with acceptable training efficiency} in DPFL can be challenging job. To deal with the challenge, the \XY{theoretical analysis} (Thm. \ref{Thm: analytical solution}) can serve as an important guidance for the \XY{low cost} parameter design of DPFL framework. Specifically, the parameter design process is structured into two distinct stages:

\begin{figure*}
    \centering
    \subfigure[Pareto Solution Fitting (LR)]{
        \includegraphics[scale=0.47]{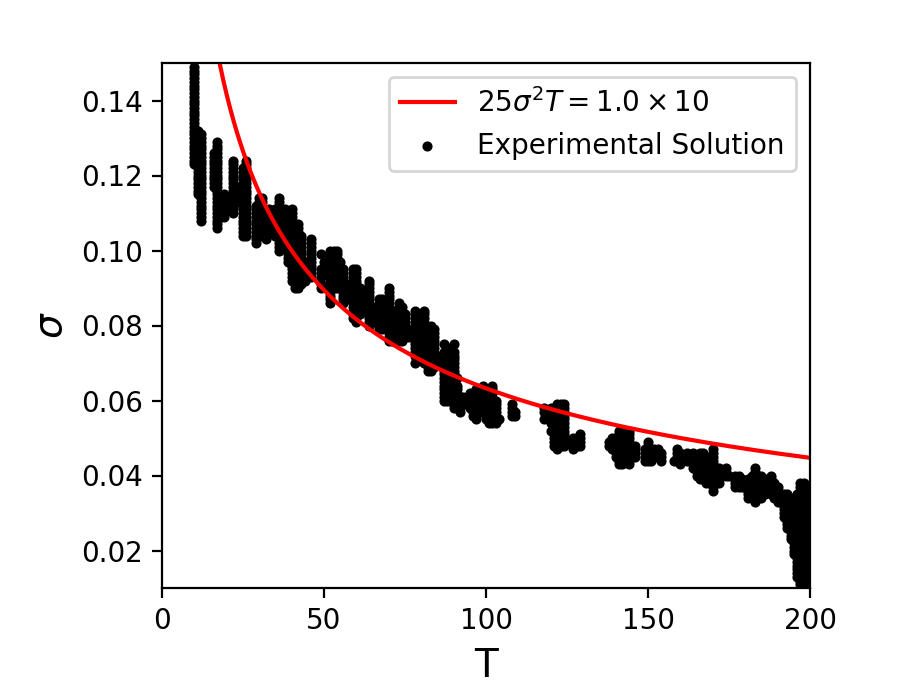} 
        \label{fig:guide_fitting_LR}
    }
    \subfigure[Pareto Solution Comparison (LR)]{
        \includegraphics[scale=0.47]{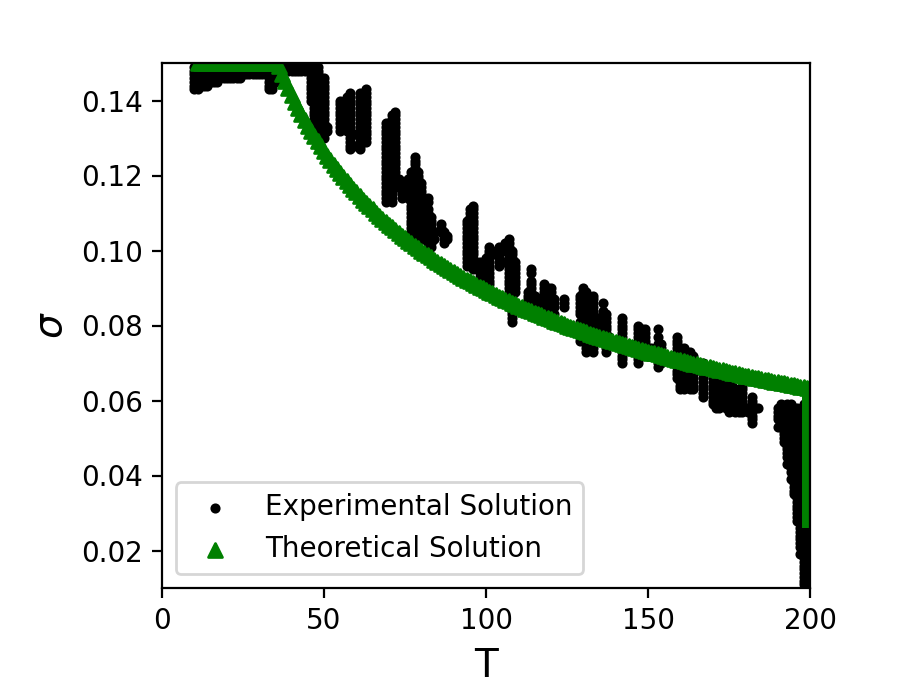} 
        \label{fig:comparison_solution_LR}
    }
    \subfigure[Pareto Front Comparison (LR)]{
        \includegraphics[scale=0.47]{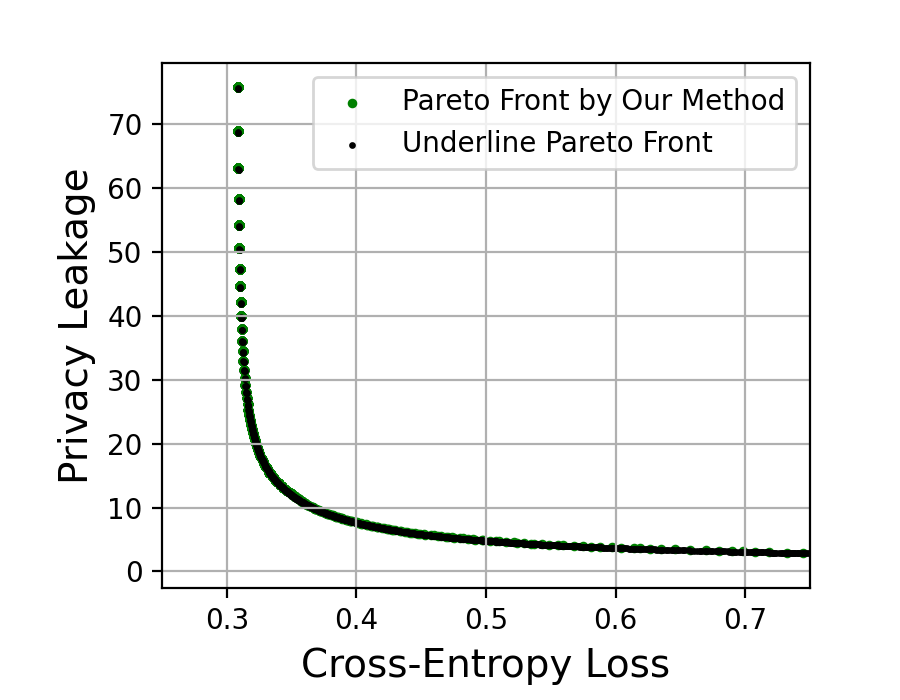} 
        \label{fig:comparison_front_LR}
    }

    \subfigure[Pareto Solution Fitting (LeNet)]{
        \includegraphics[scale=0.47]{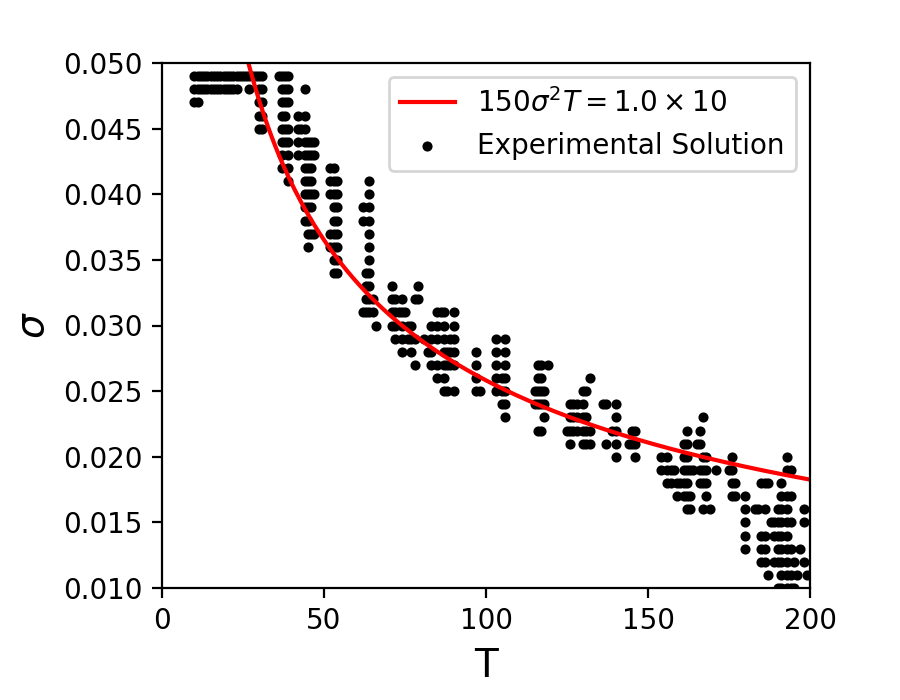} 
        \label{fig:guide_fitting_LeNet}
    }
    \subfigure[Pareto Solution Comparison (LeNet)]{
        \includegraphics[scale=0.47]{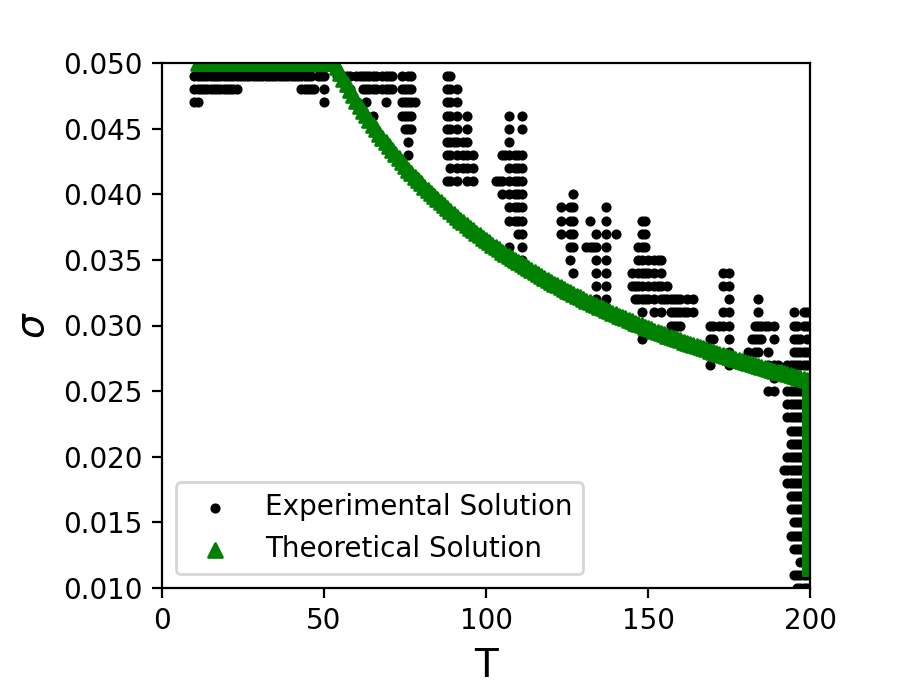} 
        \label{fig:comparison_solution_LeNet}
    }
    \subfigure[Pareto Front Comparison (LeNet)]{
        \includegraphics[scale=0.47]{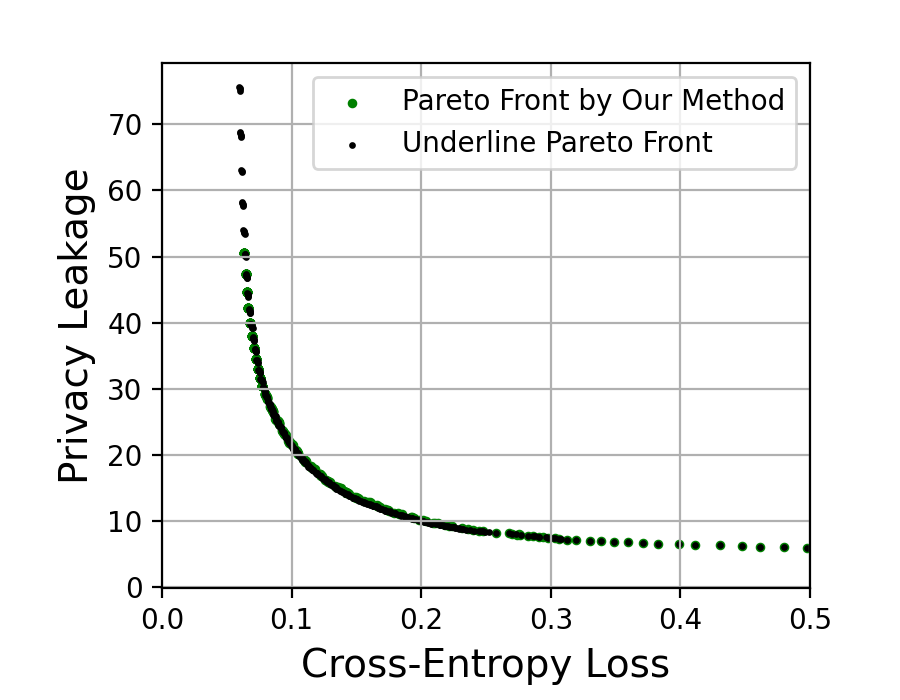} 
        \label{fig:comparison_front_LeNet}
    }
    \DeclareGraphicsExtensions.
    \caption{\textbf{Parameter Design Guidance by Public Dataset.}
    With $K=10$ and $q=1.0$, Fig. \ref{fig:guide_fitting_LR} shows the Pareto solution on public dataset and the solutions are fitted by the line $25 \sigma^2 T = 1.0 \times 10$. Fig. \ref{fig:comparison_solution_LR} compares the derived theoretical Pareto solution from public dataset and the underline experimental Pareto solution of case $K=40$ and $q=0.5$. Fig. \ref{fig:comparison_front_LR} compares the derived theoretical Pareto front from public dataset and the underline experimental Pareto front of case $K=40$ and $q=0.5$.
    With $K=10$ and $q=1.0$, Fig. \ref{fig:guide_fitting_LeNet} shows the Pareto solution on public dataset and the solutions are fitted by the line $150 \sigma^2 T = 1.0 \times 10$. Fig. \ref{fig:comparison_solution_LeNet} compares the derived theoretical Pareto solution from public dataset and the underline experimental Pareto solution of case $K=40$ and $q=0.5$. Fig. \ref{fig:comparison_front_LeNet} compares the derived theoretical Pareto front from public dataset and the underline experimental Pareto front of case $K=40$ and $q=0.5$.
    }
    \label{fig: Public dataset guide parameter design}
\end{figure*}

% Specifically, the parameter design is divided into two steps: 
% \begin{itemize}
%     \item with a small portion of public dataset and a fixed sample ratio $q_0$ and the number of total clients $K_0$, the client do pre-experiments as Algo. \ref{} to look for the Pareto optimal solution corresponding to the noise level $\sigma$ and global epoch $T$ (black points in Fig. \ref{}). Then the client simulates the black points by $\sigma^2T=q_0K_0/k$ (red line in Fig. \ref{}) and further estimate the constant $k$\footnote{Noted that the constant $k$ of the Pareto solution for two parameters $\sigma$ and $T$ is the same as the three parameters $q, T$ and $\sigma$ according to Cor. \ref{}.};
%     \item  The server distributes the common sample ratio $q_r$ and gloabl training epoch $T_r$ to all clients, and each client calculate the noise level $\sigma_r = \sqrt{\frac{q_rK}{kT_r}}$.
% \end{itemize} 

\begin{enumerate}
\item In the first step, using a small portion of a public dataset and keeping the sample ratio fixed at $q_0$, \XY{given the acceptable training efficiency constraint}, along with the total number of clients set to $K_0$, clients undertake pre-experiments, following the procedure outlined in Algo. \ref{algorithm: non-dominated sorting}. These experiments aim to identify the Pareto optimal solutions corresponding to the noise level ($\sigma$) and the global epoch ($T$), represented as the black points in Fig. \ref{fig:guide_fitting_LR} and \ref{fig:guide_fitting_LeNet}.
Subsequently, clients extrapolate the behavior of these black points by utilizing the relationship $k\sigma^2T = q_0K_0$ (depicted as the red line in Fig. \ref{fig:guide_fitting_LR} and \ref{fig:guide_fitting_LeNet}). This extrapolation further enables the estimation of the constant $k$\footnote{It's important to note that this constant $k$ remains consistent whether considering two parameters, $\sigma$ and $T$, or three parameters, namely, $q$, $T$, and $\sigma$, as indicated in Cor. \ref{Cor: analytical solution with given q}.}.
\item 
In the second step, the server uniformly distributes \XY{the common sample ratio $q_r$ and global training epoch ($T_r$)} to all participating clients.
Each client then calculates their respective noise level, denoted as $\sigma_r$, based on the formula $\sigma_r = \sqrt{\frac{q_rK}{kT_r}}$.
\end{enumerate}

Take two examples for LeNet and LR model, let the number of clients ($K_r$) and sample ratio ($q_r$) be $40$ and $0.5$ respectively. For the public dataset, we randomly sample $10\%$ of original MNIST dataset. 
With $q_0=1.0$ and $K_0=10$, the constant $k$ is estimated to be $2.5$ and $15.0$ for LR and LeNet model separately as illustrated in Fig. \ref{fig:guide_fitting_LR} and \ref{fig:guide_fitting_LeNet}. 
Then, to validate the correctness of the estimated constant $k$, we compare the theoretical solutions and experimental solutions of LR and LeNet model shown in Fig. \ref{fig:comparison_solution_LR} and \ref{fig:comparison_solution_LeNet} respectively. 
We derive the theoretical solutions by using the consistent constant $k$ and represent the solutions by green points. The black points represents the experimental solutions, which are obtained via Algo. \ref{algorithm: non-dominated sorting} using all MINTS data.
Figure \ref{fig:comparison_front_LR} and \ref{fig:comparison_front_LeNet} demonstrate that these black points align with the theoretical Pareto front obtained using the estimated values of $k$.

We compare the computation cost (the time of guiding parameter design and the time of achieving Pareto set) of our method with the existing two type of methods. The first method named as  \textbf{Training with Budget \cite{geyer2017differentially,wei2020federated}} is to minimize the utility loss when the privacy leakage is below the specified privacy budget.
The second method named as \textbf{Training until Convergences \cite{yang2019federatedconceptsandapplications,wu2020theoretical}} is to minimize the utility loss until the model convergence with different noise level and identify the optimal noise level with the least privacy leakage. 

Table \ref{table:comparison1} show the time to get the optimal parameter design with given $T$ and $q$ from server.
For our proposed method, the process of guiding parameter design can be divided into two parts: 1) The pre-experiment time denoted as $t_0$ is much smaller than the formal training time since the public dataset using in pre-experiment is much smaller; 2) The main training time is global training epoch $T_r$ since the clients directly calculate the optimal $\sigma$ with given $q_r$ and $T_r$ according to Theorem \ref{Cor: analytical solution with given q}.  
For the method \textbf{Training with Budget} and \textbf{Training Until Convergence}, they needs to search the optimal $\sigma$, thus, its computation complexity is $\Theta(n_{\sigma}T_r)$, where $n_{\sigma}$ is the number of different $\sigma$. 

Similarly, Table \ref{table:comparison2} contrasts the time required to locate the entire Pareto set under different methods. It reveals that our proposed method reduces the search time by $n_q$
times compared to the two methods, \textbf{Training with Budget} and \textbf{Training Until Convergence}, i.e., it eliminates the need to search over $q$. This efficiency is attributed to Theorem \ref{Thm: analytical solution}, which aids in determining the sample ratio $q$ given $\sigma$ and $T$.
% The total time of the proposed method can be divided into two parts as pre-experiment and main experiment. 
% The pre-experiment time denoted as $t_1$ is much smaller than the formal training time as the public dataset using in pre-experiment is much smaller.
% According to Remark \ref{remark}, the proposed method can reduce the searching space from three dimensions as $\sigma,T, q$ to two dimensions as $\sigma, T$ or $q,T$.
% Based on that, the computation complexity of main training is reduced to $\Theta(n_{\sigma}T_r)$ or $\Theta(n_{q}T_r)$, where $n_{\sigma}$ is the number of different $\sigma$ and $n_q$ is the number of different $q$.
% For the methods "Training with Budget" and "Training with convergence", they both needs to search over all $\sigma$, $T$, and $q$. The computation complexity of "Training with Budget" and "Training with convergence" is $n_qn_{\sigma}T_r$ and $n_qn_{\sigma}T_r^{\prime}$ respectively, where $T_r^{\prime}$ represents a large number of communication rounds.

\begin{table}[htbp]
\caption{Complexity Comparison (Guiding Parameter Design)}
    \begin{center}
        \begin{tabular}{lccc}
            \hline
            Method & Computational Complexity \\ 
            \hline 
            Our Method & $t_0 + \Theta(T_r)$ \\
            Training with Budget \cite{geyer2017differentially,wei2020federated} & $\Theta(n_{\sigma}T_r)$ \\
            Training until Convergence \cite{yang2019federatedconceptsandapplications,wu2020theoretical} & $\Theta(n_{\sigma} T_r)$ \\
            \hline
        \end{tabular}
    \end{center}
    \label{table:comparison1}
\end{table}

\begin{table}[htbp]
\caption{Table of Complexity Comparison (Achieving Pareto Set)}
    \begin{center}
        \begin{tabular}{lccc}
            \hline
            Method & Computational Complexity \\ 
            \hline 
            Our Method & $t_0 + \Theta(n_{\sigma}T_r)$ \\
            Training with Budget \cite{geyer2017differentially,wei2020federated} & $\Theta(n_{q}n_{\sigma}T_r)$ \\
            Training until Convergence \cite{yang2019federatedconceptsandapplications,wu2020theoretical} & $\Theta(n_{q} n_{\sigma}T_r)$ \\
            \hline
        \end{tabular}
    \end{center}
    \label{table:comparison2}
\end{table}

\section{Conclusion and Discussion}

\XY{Bearing in mind the aim of achieving optimal privacy-utility trade-off within an acceptable training efficiency constraint, we formulate the constrained bi-objective optimization formulation in Differential Privacy Federated Learning (DPFL).}
\XY{The theoretical analysis of the constrained bi-objective optimization problem} can serve as an important guidance of the parameter design in DPFL, which could help us get rid of the expensive neural network training and federated system evaluation \cite{deb2002fast, lin2022pareto, kang2023optimizing}. 
By using a small proportion of public data, we can get an approximate estimation of the exact \XY{relationship} among $T$, $\sigma$, and $q$.

In future, we can also do similar theoretical analysis for different protection mechanisms in the federated learning framework, as long as it has or could be given a good enough upper-bound of utility loss.

% \begin{equation}
%     e^{-\epsilon}\leq \frac{f_{w_r}(x)}{f_{w}(x)} \leq e^{\epsilon}
% \end{equation}
% \input{section/appendix}

\bibliographystyle{IEEEtran}
\bibliography{ref}

% \appendix

\end{document}